\documentclass{article}

    \usepackage[preprint]{neurips_2024}

\usepackage[utf8]{inputenc} 
\usepackage[T1]{fontenc}    
\usepackage{hyperref}       
\usepackage{url}            
\usepackage{booktabs}       
\usepackage{amsfonts}       
\usepackage{nicefrac}       
\usepackage{microtype}      
\usepackage{xcolor}         

\usepackage{xcolor}
\usepackage{amssymb}
\usepackage{amsmath}
\usepackage{amsthm}
\usepackage{xspace}
\allowdisplaybreaks 
\usepackage{graphicx}
\usepackage{subcaption}

\newcommand{\probName}{\texttt{SOOTT}\xspace}

\newcommand{\algName}{\texttt{BEST}\xspace}
\newcommand{\informedGreedy}{\texttt{IGA}\xspace}
\newcommand{\greedy}{\texttt{PGA}\xspace}
\newcommand{\cort}{\texttt{CoRT}\xspace}

\newtheorem{assumption}{Assumption}
\newtheorem{theorem}{Theorem}[section]
\newtheorem{remark}{\textbf{Remark}}[section]

\newtheorem{lemma}[theorem]{Lemma}
\newtheorem{definition}[theorem]{Definition}
\newtheorem{proposition}[theorem]{Proposition}
\usepackage{amssymb}
\usepackage{xcolor}
\usepackage[linesnumbered,ruled,vlined]{algorithm2e}

\newcommand{\argmin}{\operatorname*{argmin}}

\title{Smoothed Online Optimization for Target Tracking: Robust and Learning-Augmented Algorithms}

\author{%
  Ali Zeynali\\
  University of Massachusetts Amherst\\
  \texttt{azeynali@cs.umass.edu} \\
  \And
  Mahsa Sahebdel \\
  University of Massachusetts Amherst\\
  \texttt{msahebdelala@umass.edu} \\
  \And
  Qingsong Liu \\
  University of Massachusetts Amherst\\
  \texttt{qingsongliu@umass.edu } \\
    \And
  Mohammad Hajiesmaili \\
  University of Massachusetts Amherst\\
  \texttt{hajiesmaili@cs.umass.edu} \\
  \And
  Ramesh~K. Sitaraman \\
  University of Massachusetts Amherst \& Akamai Technologies\\
  \texttt{ramesh@cs.umass.edu} \\
}

\begin{document}

\maketitle

\begin{abstract}
We introduce the \textit{Smoothed Online Optimization for Target Tracking} (\probName) problem, a new framework that integrates three key objectives in online decision-making under uncertainty: (1) \textit{tracking cost} for following a dynamically moving target, (2) \textit{adversarial perturbation cost} for withstanding unpredictable disturbances, and (3) \textit{switching cost} for penalizing abrupt changes in decisions. This formulation captures real-world scenarios such as elastic and inelastic workload scheduling in AI clusters, where operators must balance long-term service-level agreements (e.g., LLM training) against sudden demand spikes (e.g., real-time inference).
We first present \algName, a robust algorithm with provable competitive guarantees for \probName. To enhance practical performance, we introduce \cort, a learning-augmented variant that incorporates untrusted black-box predictions (e.g., from ML models) into its decision process. Our theoretical analysis shows that \cort strictly improves over \algName when predictions are accurate, while maintaining robustness under arbitrary prediction errors.
We validate our approach through a case study on workload scheduling, demonstrating that both algorithms effectively balance trajectory tracking, decision smoothness, and resilience to external disturbances.
\end{abstract}

\section{Introduction}
\label{sec:intro}
This paper introduces and studies the \textit{Smoothed Online Optimization for Target Tracking} problem (\probName), a new framework that captures three interacting objectives in online target tracking. At each round, an agent selects an action, evaluated based on the alignment of a windowed average of its recent actions with a dynamically moving target. The agent incurs three types of costs: (1) a tracking cost penalizing deviations between the agent’s time-averaged action and a desired but dynamically moving target, (2) an adversarial perturbation cost reflecting external disturbances that is unpredictable and arriving online, and (3) a switching cost that penalizes abrupt changes in the agent's decisions. Together, these components form a composite loss that challenges conventional online optimization techniques by introducing dependencies on both historical behavior and adversarial adjustments. Effective minimization of this loss requires algorithms capable of balancing smooth trajectory alignment, smoothness in decision-making, and resilience to adversarial disturbances.

A key motivational application for studying \probName\ arises from the need to efficiently manage the scheduling of \textit{elastic} workloads (e.g., AI training) and \textit{inelastic} workloads (e.g., AI inference) in large-scale cloud and AI clusters~\cite{berg2020optimal,Li_elastic_deepLearning}.
. In such environments, operators must simultaneously maintain target processing rates for long-running elastic jobs to meet the customer service-level agreements (SLAs) while accommodating unpredictable spikes in latency-sensitive inelastic workload. This dual demand requires dynamic, on-the-fly resource reallocation, where elastic jobs (e.g., LLM training, finance analysis, software maintenance, and upgrade) can be paused or throttled to prioritize inelastic jobs (e.g., real-time AI inference). At each decision epoch and \textit{before the realization of the inelastic demand}, operators must judiciously determine what fraction of resources to allocate to inelastic jobs, leaving the remainder for elastic ones. Over-allocating (allocating resources more than the realized demand) to inelastic jobs risks leaving resources idle and failing to meet the SLA of elastic jobs, while under-allocating can result in unserved inference requests~\cite{chatGPT_cap, chatGPT_cap2}.
Additionally, this flexibility in resource allocation between elastic and inelastic workload comes at a cost: frequent pause and resuming multi-hundred-gigabyte training workloads impose heavy checkpoint-and-restore penalties, making reckless preemption highly counterproductive~\cite{lechowicz2024online,lechowicz2023online}, and \probName captures this by adding the switching cost terms. This motivates our \probName framework, which captures these trade-offs explicitly: (1) the sliding window tracking term that models long-term SLA requirement for elastic jobs over time, (2) the adversarial perturbation term represents bursty or unpredictable demands in inelastic jobs that can only be observed \textit{after} resource allocation; and (3) the switching cost accounts for the overhead of frequent changes in resource (re)allocation.

Beyond the main case study of the elastic and inelastic workload scheduling, \probName is a general framework that is well-suited for broader range of applications, e.g., server maintenance scheduling~\cite{hu2008online,liaskos2012broadcast}, where consistent service requires regular interventions over a time window; image-based object tracking~\cite{cai2022memot,song2024moviechat,gao2023memotr}, where predictions must remain coherent across frames,  online control~\cite{zhang2022adversarial,zhao2023non} where system stability and performance depend on sequences of past inputs, and resource pooling in shared infrastructures, e.g., in multi-tenant cloud platforms and shared Electric Vehicle (EV) charging platforms~\cite{macdonell2023infrastructure}. In resource pooling of shared infrastructures, operators dynamically allocate shared resources—such as compute, bandwidth, or energy—among multiple users or applications with varying demand profiles and SLA requirements. The challenge is to maintain fair and efficient resource distribution in the presence of unpredictable and non-stochastic workload. Our model naturally captures the need for smooth adjustments while mitigating abrupt disruptions in the allocations.

On the theory front, our framework brings together two well‑studied strands of online optimization for target tracking that have so far evolved largely in parallel: (1) memory-based online tracking, where past actions influence current tracking cost~\cite{lin2021perturbation,lin2024online,zhang2022adversarial}; and (2) smoothed online optimization~\cite{anava2015online,shi2020online,zhang2022adversarial}  which penalizes abrupt changes in decision-making. We provide a comprehensive review of the related literature in the Appendix~\S\ref{sec:background} and highlight how existing algorithms fail to solve our problem holistically. Specifically, existing methods either neglect the memory-based dynamics introduced by the sliding window tracking term or significantly simplify them, or overlook the role of the smoothness component. In this paper, we develop algorithms for \probName under competitive worst-case analysis (i.e., without assuming any predictions of adversarial perturbations and dynamic target) and aim to develop algorithms that achieve a solid constant \textit{competitive ratio}, defined as the worst-case ratio between the cost of an online algorithm and the offline optimum~\cite{borodin1992optimal,manasse1988competitive}.

While worst-case guarantees offer robustness, they may be overly conservative or cautious. In recent years, learning-augmented online algorithms~\cite{lykouris2021competitive,purohit2018improving} have emerged to use potentially imperfect predictions to achieve two goals: performing near-optimally when the predictions are accurate (i.e., \emph{consistency}) and retaining worst-case guarantees when predictions are misleading (i.e., \emph{robustness}). 
These algorithms bridge the gap between worst-case guarantees and practical performance by incorporating untrusted predictions. However, applying this to our setting introduces unique challenges. Unlike classical online models where predictions are straightforward (e.g., demand or price forecasts), the interplay between sliding-window tracking, adversarial perturbation, and switching costs creates complex interdependencies. As a second goal of this paper, we aim to propose learning-augmented algorithms that effectively integrate machine-learned advice to enhance practical performance while retaining robustness against erroneous predictions.

\textbf{Main contributions.} We study the problem of smoothed online optimization for target tracking, denoted as \probName, where the objective is to minimize a cost function including three components: (1) tracking cost, (2) adversarial cost, and (3) switching cost. 
We provide both robust and learning-augmented algorithms for \probName and the key technical contributions are summarized below.

\emph{Competitive analysis.} We begin by presenting \informedGreedy, a semi-online algorithm that has access to the adversary’s exact target for the current time step, but not for future. Through a competitive analysis, we establish a constant upper bound on its competitive ratio. Building on this, we propose \algName, a fully online algorithm for \probName, and analyze its worst-case performance by bounding its cost relative to that of \informedGreedy. Furthermore, we demonstrate the tightness of our competitive guarantees by showing that it recovers the best-known results in relevant special cases~\cite{shi2020online,goel2019beyond}.

\emph{Learning-augmented analysis.} To improve performance beyond worst-case guarantees, we consider the learning-augmented setting. We begin with a natural baseline algorithm that fully trusts predictions; while it performs near-optimally with perfect predictions, it is fragile under adversarial noise and lacks robustness in such cases. To address the lack of robustness, we propose \cort, a robust learning-augmented algorithm that leverages predictions to improve over \algName when they are accurate, while still retaining competitive guarantees under worst-case conditions. Our analysis reveals a fundamental trade-off in \cort between its consistency and robustness, which can be tuned via a controllable algorithmic parameter.

\emph{Case study.} Using real-world traces from Google Cloud~\cite{GCD_dataset}, we empirically evaluate our algorithms through a case study on dynamic resource allocation for both elastic and inelastic workloads. Notably, our experiments demonstrate that the performance of the \cort algorithm closely approaches that of \informedGreedy—our proposed semi-online but impractical algorithm that assumes perfect knowledge of online inputs—while also maintaining robustness against arbitrarily inaccurate predictions.

\emph{Technical novelty.} 
Our analysis builds on two new ingredients: (1) we leverage the fact that the auxiliary objective $g_t(u)$(Lemma~\ref{lem:phi_chi_convexity}) is strongly convex, and together with a Lipschitz‑stability result for the windowed minimiser (Lemma \ref{lem:sum_diff_ys}), to achieve a \emph{two‑level contraction} that simultaneously reduces the prediction gap and the accumulated history error. (2) We develop a bespoke sliding‑window Cauchy–Schwarz lemma (Lemma \ref{lem:sum_chi}) to convert convoluted memory sums into point‑wise norms while preserving tight constants. These tools drive the tight bounds for \algName and the consistency–robustness guarantee of our learning‑augmented \cort algorithm.



\section{Problem Formulation}
\label{sec:ProbForm}

\noindent\textbf{Model and problem statement.}
We consider the problem of \emph{smoothed online optimization for target tracking} (\probName) where an agent chooses an action at each time step under an adversarial perturbation setting. At each time $t \in \mathbb{N}$, a trajectory target point $\tau_t \in \mathbb{R}^d$ and a time-varying adversarial perturbation function $f_t : \mathbb{R}^d \rightarrow \mathbb{R}_{\geq 0}$ are revealed to the agent. Meanwhile, the adversary selects a hidden target $u_t \in \mathbb{R}^d$, which is disclosed only after the agent has chosen action $x_t \in \mathbb{D} \subset \mathbb{R}^d$, where $\mathbb{D}$ is a compact set representing the domain of valid actions.  The agent then incurs the following cost:
\begin{equation} 
    \text{Cost}_t(x_t, h_t) = 
    \underbrace{ \left\| \frac{x_t + h_t}{w + 1} - \tau_t \right\|^2}_{\text{tracking cost}} + 
    \underbrace{\lambda_1 f_t(x_t - u_t)}_{\text{adversarial cost}} + 
    \underbrace{{\lambda_2} \|x_t - x_{t-1}\|^2}_{\text{switching cost}},
\end{equation}
where $h_t = \sum_{i=1}^{w} x_{t-i}$ represents the aggregation of the agent’s past $w$ actions, and $\lambda_1, \lambda_2 > 0$ are fixed weighting parameters. The goal of the
agent is to select actions that minimize the cumulative cost over $T$ time steps: $\sum^T_{t=1} \text{Cost}_t(x_t, h_t)$. 

This cost function captures three competing objectives. The first term penalizes deviations between the agent's time-averaged action over the current and past $w$ rounds and a desired trajectory target $\tau_t$, thereby encouraging tracking the moving target. The second term, $\lambda_1 f_t(x_t - u_t)$, reflects the adversarial influence and penalizes discrepancies between the agent’s action and the (hidden) target of the adversary, $u_t$, through a function $f_t$. The third term, ${\lambda_2} \|x_t - x_{t-1}\|^2$, imposes a regularization that discourages abrupt changes in the agent’s behavior over consecutive slots, promoting smoothness in the sequence of actions. The trade-offs between these objectives are governed by the parameters $\lambda_1$ and $\lambda_2$. 

To enable tractable analysis, we impose the following standard structural assumptions on the adversarial perturbation and initialization:

\begin{assumption}[Adversarial Minimum]
    \label{asmp:min_val_adv}
    The time-dependent adversarial perturbation function $f_t(\cdot)$ is non-negative and minimized at the origin without loss of generality, i.e., $\arg\min f_t(x) = \textbf{0}$. 
\end{assumption}

\begin{assumption}[Adversarial Convexity]
    \label{asmp:convexity}
    Function $f_t(\cdot)$ is $m$-strongly convex for some $m > 0$.
\end{assumption}

\begin{assumption}[Adversarial Smoothness]
    \label{asmp:smoothness}
    Function $f_t(\cdot)$ is $\ell$-smooth, meaning its gradient is $\ell$-Lipschitz continuous for some $\ell > 0$.
\end{assumption}

\begin{assumption}[Initial Convergence]
    \label{asmp:start_history}
    For all $t \leq 0$, the agent's actions and the adversary's targets are both initialized at the origin. Additionally, the online algorithm coincides with the offline optimal strategy over this initial period.
\end{assumption}

These assumptions are standard online optimization literature and allow for meaningful theoretical analysis~\cite{shi2020online,zhang2022smoothed,li2020leveraging,zhao2020dynamic}. Assumption~\ref{asmp:min_val_adv} ensures that the adversarial cost cannot reward the agent through negative values and is minimized when the agent exactly matches the target of the adversary. Assumptions~\ref{asmp:convexity} and~\ref{asmp:smoothness} impose structure to the adversarial perturbations, enabling gradient-based analysis. Finally, Assumption~\ref{asmp:start_history} provides a synchronized and consistent initialization for the online optimization process.

\noindent\textbf{Challenges.}
A major key challenge in \probName arises from the presence of memory, i.e., term $h_t$, and $x_{t-1}$, that includes the historical actions,  in the cost function, which introduces temporal coupling across decisions. Specifically, the agent's current cost depends not only on its present action but also on a window of past actions. Prior work~\cite{shi2020online,goel2019beyond,cha2024bandit} has demonstrated that, even in an idealized setting where the agent has full knowledge of the target of the adversary, $u_t$, before committing to an action, identifying the optimal decision remains nontrivial due to this memory dependency. Notably, when the influence of memory is limited—e.g., when the memory window $w$ and the smoothness regularization coefficient $\lambda_2$ are sufficiently small—the problem becomes effectively myopic. In such cases, a greedy strategy that minimizes the instantaneous cost can closely approximate the optimal offline policy which has a full knowledge of future input, i.e., $\{\tau_t,\ u_t\}_{t=1}^T$, and adversarial cost functions $\{f_t\}_{t=1}^T$.
The second challenge in \probName stems from the fact that the target of the adversary $u_t$ is revealed only after the agent has selected $x_t$. Thus, the adversarial cost term is not directly observable at the time of decision, which complicates the design of algorithms with guaranteed performance.

\noindent\textbf{Competitive analysis.}
Our goal is to design an online algorithm that guarantees a small competitive ratio~\cite{borodin1992optimal,manasse1988competitive} which guarantees performing near optimal offline algorithm. Formally, for an online algorithm $ALG$ and an input instance $\mathcal{I}$, the competitive ratio is defined as: $\texttt{CR}(ALG) = \sup_{\mathcal{I}} \ {\text{Cost}(ALG, \mathcal{I})}/{\text{Cost}(OPT, \mathcal{I})}$, 
where Cost($ALG, \mathcal{I}$), and  Cost($OPT, \mathcal{I}$) denote the cost of algorithm $ALG$ and offline optimum on instance $\mathcal{I}$. In addition, to further simplify the presentation of theoretical bounds, in this paper we use the \emph{degradation factor} (\texttt{DF}) metric, introduced in~\cite{zeynali2021data}, to bound the worst-case ratio between the performance of two online algorithms. Specifically, the degradation factor of algorithm $ALG_1$ relative to another algorithm $ALG_2$ is defined as: $\texttt{DF}(ALG_1, ALG_2) = \sup_{\mathcal{I}} \ {\text{Cost}(ALG_1, \mathcal{I})}/{\text{Cost}(ALG_2, \mathcal{I})},$
 which also implies an upper bound on the competitive ratio of $ALG_1$ in terms of that of $ALG_2$: $\texttt{CR}(ALG_1) \leq \texttt{DF}(ALG_1, ALG_2) \cdot \texttt{CR}(ALG_2).$ When $ALG_2$ is the optimal offline algorithm, the degradation factor coincides with the competitive ratio of $ALG_1$.

\vspace{-2mm}
\section{Robust Online Algorithms for \probName}
\label{sec:online_algorithms}
\vspace{-3mm}

In this section, we introduce \informedGreedy, a semi-online benchmark algorithm that relaxes the uncertainty of $u_t$ by assuming that the adversary's target $u_t$ is known at the current time step, but remains unknown for future time steps. Although this assumption is unrealistic in most practical scenarios, \informedGreedy plays an important analytical role, serving as a performance baseline against which we compare more practical algorithms that do not have access to this information. Then, in Section~\ref{sec:BEST}, we present \algName, a fully online algorithm that operates without knowledge of the adversary’s target and analyze its performance by bounding its degradation factor relative to \informedGreedy.

\subsection{\informedGreedy: A Semi-online Benchmark Algorithm with Exact Knowledge of $u_t$}
\label{sec:IGA}
This section introduces \underline{I}nformed \underline{G}reedy \underline{A}lgorithm (\informedGreedy), which known $u_t$ when taking its action. Given this additional information, \informedGreedy selects actions that greedily minimize the cost function at each time step. This setting captures an idealized scenario in which the adversary’s intention is perfectly predictable and the cost structure is fully known in advance. Although such assumptions may not hold in practice, the performance of \informedGreedy offers a meaningful baseline to assess the quality of practical online algorithms.


At each time step $t$, \informedGreedy observes the target of the adversary $u_t$ and chooses an action $x_t$ that minimizes the total cost over the current time step, balancing target tracking, adversarial deviation, and switching penalties. The pseudo-code of \informedGreedy is presented in Algorithm~\ref{alg:informed_greedy}.

\begin{algorithm}[H] 
    \label{alg:informed_greedy}
  \SetAlgoLined
  \KwData{ $\hat{x}_{t-w:t-1}$, $u_{t}$, $\tau_t$}
  \KwResult{$\hat{x}_t$: action of the \informedGreedy at time $t$} 

  $\hat{x}_{t} \leftarrow \argmin_{x} ||\frac{x + \hat{h}_{t}}{w+1} - \tau_{t}||^2 + {\lambda_1} \ f_{t}(x - u_{t}) +  {\lambda_2} ||x - \hat{x}_{t-1} ||^2$

 \textbf{Output:} $\hat{x}_{t}$
  \caption{The Informed Greedy Algorithm (\informedGreedy)}
\end{algorithm}

Since \informedGreedy has the full knowledge of the cost function at time step~$t$, it can select the action that minimizes the cost at that time step, given the current history~$h_t$. However, as it lacks foresight into future target values $\tau_t$ and target of the adversary $u_t$, its chosen actions may still diverge from those of the optimal offline solution. The following result establishes a performance guarantee for \informedGreedy in terms of its competitive ratio, evaluated against the cost incurred by the optimal offline strategy.

\begin{theorem}
\label{thm:CR_perfectPrediction}
If $2w^2< 2 + m \lambda_1 (w+1)^2$, the competitive ratio of \emph{\informedGreedy} is upper bounded by
\begin{equation}
    \emph{\texttt{CR}}(\emph{\informedGreedy} ) \leq 1 + \frac{2(\lambda_2\ (w+1)^2 + w^2)}{m \lambda_1 (w+1)^2 + 2 - 2w^2}.
\end{equation}
\end{theorem}
\vspace{-2mm}
The proof of Theorem~\ref{thm:CR_perfectPrediction} is given in Appendix~\S\ref{app:CR_perfectPrediction_proof}. When both $\lambda_1$ and $\lambda_2$ are large, the setting reduces to standard smoothed online convex optimization~\cite{goel2019beyond, shi2020online, li2023robustified}, and the competitive ratio of \informedGreedy converges to $1 + \frac{2\lambda_2}{m\lambda_1}$, consistent with results in the literature~\cite{shi2020online,goel2019beyond}. When $w > 0$, the optimal offline algorithm considers future consequences of current actions, while \informedGreedy makes locally optimal decisions using perfect knowledge of $u_t$. In such cases, when $\lambda_1$ is small or $f_t$ is weakly convex, the impact of adversarial cost is diminished, and \informedGreedy may perform arbitrarily worse than the offline optimum—hence the necessity of the condition in Theorem~\ref{thm:CR_perfectPrediction} to ensure bounded competitive ratio.

In the remainder of the paper, we develop online algorithms without perfect information of $u_t$ and assess their performance using the \textit{degradation factor} metric with respect to \informedGreedy. This allows us to derive meaningful performance guarantees relative to the offline optimum by combining the bounds on the {degradation factor} with the result of Theorem~\ref{thm:CR_perfectPrediction}.

\subsection{\algName: A Robust Algorithm for \probName}
\label{sec:BEST}

We present \underline{B}ackward \underline{E}valuation for \underline{S}equential \underline{T}argeting (\algName), an online algorithm for \probName that does not require any knowledge of $u_t$ in the current and future time step. Since online algorithms lack exact information about the adversary’s target, a naive greedy approach (which is also blind to $u_t$) that minimizes the cost at each time step independently can diverge significantly from the behavior of the \informedGreedy, leading to substantially higher cumulative costs. Our algorithm is designed to keep its actions close to the actions of the \informedGreedy by considering the history of \informedGreedy during its action selection process. \algName ignores the adversarial cost term in its own historical actions, and selects its action based on the history of \informedGreedy. Note that the history of \informedGreedy is accessible to \algName since, after observing the target of the adversary in each time step, one could exactly calculate the corresponding action of \informedGreedy.
We present the pseudo-code of \algName in Algorithm~\ref{alg:worst_opt}.

\begin{algorithm}[H] 
    \label{alg:worst_opt}
  \SetAlgoLined

  \KwData{ $u_{t-1}$, $\tau_t$, $\hat{x}_{t-w-1:t-2}:$ history of actions taken by \informedGreedy}
  \KwResult{$x_t$: action of the agent at time $t$} 

  $\hat{x}_{t-1} \leftarrow \text{action of } \informedGreedy \text{ at time step } (t-1)$

    
  $x_t \quad \leftarrow  \argmin_{x}  ||\frac{x + \hat{h}_{t}}{w+1} - \tau_{t}||^2 + {\lambda_2} ||x - \hat{x}_{t-1}||^2$

 \textbf{Output:} $x_t$
  \caption{Backward Evaluation for Sequential Targeting (\algName)}
\end{algorithm}

In Line 1, \algName finds the action of \informedGreedy in the previous time step, $\hat{x}_{t-1}$, since the most recent target of the adversary has already been revealed. It keeps track of the history of action taken by \informedGreedy and evaluates $\hat{x}_{t-1}$, and $\hat{h}_t$ based on \informedGreedy's past actions. Next, in Line 2, \algName observes the current trajectory target $\tau_t$ and selects its action by ignoring the adversarial cost term and assuming its history matches that of \informedGreedy. Note that if the target of the adversary at time step $t$, is different from $x_t$, the action of \algName, $x_t$, and the action of \informedGreedy, $\hat{x}_t$ will be different. Due to this difference, the cost value incurred by \informedGreedy at time step $t$ would be proportional to $||x_t - \hat{x}_t||^2$ since all terms in the cost function are convex and smooth. This insight shows how  \algName keeps its cost values close to the cost of \informedGreedy algorithm which formally analyzed in the following Theorem. The following Theorem shows that \algName achieves a bounded degradation factor with respect to \informedGreedy proving its worst-case performance guarantee when combined with the result of Theorem~\ref{thm:CR_perfectPrediction}. 

\begin{theorem}
\label{thm:bestAlg_CR}
The degradation factor of \emph{\algName} with respect to \emph{\informedGreedy} is bounded as follows:
\begin{align}
    \emph{\texttt{DF}}(\emph{\algName}, \emph{\informedGreedy}) \leq   1 + \frac{\ell}{m}\cdot \frac{\eta^2 + 2\lambda_1 \ell (1+\lambda_2)}{\eta (\eta - m\lambda_1)}.
\end{align}
where $\eta = 2/(w+1)^2 + m\ \lambda_1 + 2\lambda_2$.
\end{theorem}
The complete proof of Theorem~\ref{thm:bestAlg_CR} is provided in Appendix~\S\ref{app:bestAlg_CR_proof}. As a sketch of the proof, we define an auxiliary function $g_t(u)$ that represents the cost incurred by \informedGreedy\ at time step $t$, assuming the adversary's target is $u$. We show that $g_t(u)$ is strongly convex, with its unique minimizer corresponding to the action selected by \algName. This structural property allows us to bound the cost difference between \algName\ and \informedGreedy\ in terms of the cost of \informedGreedy\ and problem-specific constants. Then, we choose the hyperparameters introduced in the analysis, to ensure that the additive constant term in the bound is negative, which guarantees that \algName\ achieves a constant degradation factor.


\begin{remark}
\label{rem:best_linear_lambda}
The degradation factor of \emph{\algName} relative to \emph{\informedGreedy} grows at most as $\mathcal{O}(m \lambda_1)$ with respect to $\lambda_1$ and $m$. This is intuitive, as increasing either parameter linearly amplifies the influence of the adversarial cost term in its objective. Since \emph{\algName} does not account for this adversarial term in its action selection policy, its performance gap relative to \emph{\informedGreedy} increases linearly with $\lambda_1$ and $m$.
\end{remark}



\vspace{-2mm}
\section{Learning-Augmented Algorithms for \probName}
\vspace{-2mm}

Learning-augmented online algorithms incorporate machine-learned predictions of future inputs to enhance classical online decision-making~\cite{lykouris2021competitive,purohit2018improving}. In \probName, the algorithm receives a prediction of the adversary's target for the upcoming time step and integrates this estimate into its action selection strategy. While accurate predictions can significantly improve performance, blindly trusting erroneous predictions may lead to degraded outcomes, especially under high noise. To account for this, the performance of learning-augmented algorithms is typically evaluated using two complementary metrics: \emph{consistency}, which captures performance under accurate predictions, and \emph{robustness}, which measures resilience against arbitrary prediction errors. Achieving both objectives simultaneously is challenging, as improving consistency often comes at the expense of robustness, necessitating careful algorithmic design to manage this trade-off \cite{purohit2018improving}.

Let $\hat{u}_t$ denote the prediction of the adversary's target for time step $t$. As discussed in Section~\ref{sec:online_algorithms}, we use \informedGreedy as a baseline to evaluate the performance of online algorithms. Based on this, we define the notions of consistency and robustness for the \probName setting as follows:

\begin{definition}
A learning-augmented algorithm for \emph{\probName} is $\alpha$-consistent if its degradation factor relative to \emph{\informedGreedy} is at most $\alpha$ under perfect predictions, i.e., when $\hat{u}_t = u_t$ for all $t$.
\end{definition}

\begin{definition}
A learning-augmented algorithm for \emph{\probName} is said to be $\beta$-robust if its degradation factor relative to \emph{\informedGreedy} is at most $\beta$ for any predicted sequence $\{\hat{u}_t\}_{t=1}^T$.
\end{definition}

In the following sections, we demonstrate that a greedy algorithm which selects actions to minimize the immediate cost achieves optimal consistency but lacks robustness. To overcome this limitation, we introduce a learning-augmented algorithm that incorporates a tunable parameter~$\theta$, allowing us to control the trade-off between consistency and robustness.

\vspace{-2mm}
\subsection{\greedy: An Algorithm with Full Trust on the Prediction}
In this section, we present the  \underline{P}rediction-based \underline{G}reedy \underline{A}lgorithm (\greedy), which greedily finds the action that is predicted to minimize the cost value during time step $t$. Given the predicted adversary's target $\hat{u}_t$ at time step $t$, \greedy fully trusts the prediction and chooses the action that leads to the lowest cost for the current time step.
The detail of \greedy is provided in Algorithm~\ref{alg:greedy_pred}. 

\begin{algorithm}[H] 
    \label{alg:greedy_pred}
  \SetAlgoLined

  \KwData{ $\tilde{x}_{t-w:t-1}$, $\hat{u}_{t}$, $\tau_t$}
  \KwResult{$\tilde{x}_t$: action of the agent at time $t$} 

  $\tilde{x}_{t} \leftarrow \argmin_{x} ||\frac{x + \tilde{h}_{t}}{w+1} - \tau_{t}||^2 + {\lambda_1} \cdot f_{t}(x - \hat{u}_{t}) +  {\lambda_2} ||x - \tilde{x}_{t-1} ||^2$

 \textbf{Output:} $\tilde{x}_t$
  \caption{Prediction-based Greedy Algorithm (\greedy)}
\end{algorithm}

Since \greedy selects its actions by fully trusting the predicted adversary's target, its performance is highly sensitive to prediction errors. When the prediction is perfect, \greedy takes the same actions as \informedGreedy, thereby achieving optimal consistency. However, with prediction errors, the cost incurred by \greedy can deviate significantly from that of the optimal offline solution. The following theorem provides a lower bound on the degradation factor of \greedy as a function of the prediction error in $\hat{u}_t$.



\begin{theorem}
\label{thm:greedy_predicted}
The degradation factor of \emph{\greedy} with respect to \emph{\informedGreedy} is lower bounded as follows:
\begin{align*}
    \emph{\texttt{DF}}(\emph{\greedy}, \emph{\informedGreedy}) \geq \frac{m}{2\ \max_t f_t(\textbf{0})} \sum_{t=1}^T \frac{\|u_t - \hat{u}_t\|^2}{T}.
\end{align*}
\end{theorem}
\vspace{-2mm}
The full proof of Theorem~\ref{thm:greedy_predicted} is provided in Appendix~\S\ref{app:greedy_predicted_lower_proof}.

\begin{remark}
\label{rem:pga_robustness}
Since $\max_t f_t(\mathbf{0})$ can be arbitrarily close to zero, and the prediction error $\|u_t - \hat{u}_t\|$ is unbounded, the cost of \emph{\greedy} can become arbitrarily large relative to \emph{\informedGreedy} in the worst case. This demonstrates that \emph{\greedy} lacks robustness when faced with inaccurate predictions.
\end{remark}

Motivated by the lack of robustness in \greedy, in what follows, we aim to design a learning-augmented algorithm that not only enhances the performance of \algName under perfect predictions but also maintains provable robustness guarantees under noisy or adversarial prediction errors. 

\subsection{\cort: A Consistent and Robust Learning-Augmented Algorithm for \probName}

We propose the \underline{Co}nsistent and \underline{R}obust \underline{T}racking algorithm (\cort), which incorporates predictions of the adversary's target $\hat{u}_t$ while providing provable robustness guarantees (see Algorithm~\ref{alg:laga} for the pseudo-code). Like \algName, \cort selects actions using the history of \informedGreedy. However, it accounts for the adversarial cost term by estimating it through a controlled target $\tilde{u}_t$, computed from $\hat{u}_t$ and constrained to lie within a distance of at most $\theta D_t$ from \algName's action (Lines 2–5). Here, $\theta$ is a tunable algorithm parameter, and $D_t$ is a dynamically adjusted bound. The algorithm initializes with $D_1 = 0$ and updates $D_t$ based on its previous value, the deviation between $u_t$ and \algName’s action, and the discrepancy between that action and $\tilde{u}_t$ (Line 7). Intuitively, \cort adapts $D_t$ to reflect the observed deviation of the actual adversary's target from \algName’s action, thereby bounding the cumulative deviation of $\tilde{u}_t$ from \algName’s action.  See Figure~\ref{fig:cort_example} for an illustration. In the limiting case, \cort recovers \algName as $\theta \to 0$.

\begin{figure}[ht]
\centering
\begin{minipage}[t]{0.69\textwidth}
  \vspace{0pt}
  \begin{algorithm}[H]
    \SetAlgoLined
    \KwData{$\hat{u}_{t}$, $\tau_t$, $D_t$, parameter $\theta$, $\hat{x}_{t-w-1:t-2}$: history of actions taken by \informedGreedy}
    \KwResult{$\tilde{x}_t$: action of the agent at time $t$}

    $x_t \leftarrow$ action of \algName at time $t$

    $\tilde{u}_t \leftarrow \hat{u}_t$

    \If{$||\hat{u}_t - x_t|| \geq \theta D_t$}
    { $\tilde{u}_t \leftarrow x_t + \theta D_t \cdot \frac{(\hat{u}_t - x_t)}{||\hat{u}_t - x_t||}$ }

    $\tilde{x}_t \leftarrow \argmin_{x} \left\| \frac{x + \hat{h}_t}{w+1} - \tau_t \right\|^2 + \lambda_1 f_t(x - \tilde{u}_t)+ {\lambda_2} \left\| x - \hat{x}_{t-1} \right\|^2$

    $D_{t+1}^2 \leftarrow D_t^2 + ||u_t - x_t||^2 - \theta^{-2}||\tilde{u}_t - x_t||^2$

    \textbf{Output:} $\tilde{x}_t$
    \caption{Consistent and Robust Tracking Algorithm (\cort)}
    \label{alg:laga}
  \end{algorithm}
\end{minipage}
\hfill
\begin{minipage}[t]{0.3\textwidth}
  \vspace{0pt} 
  \centering
\includegraphics[width=\linewidth]{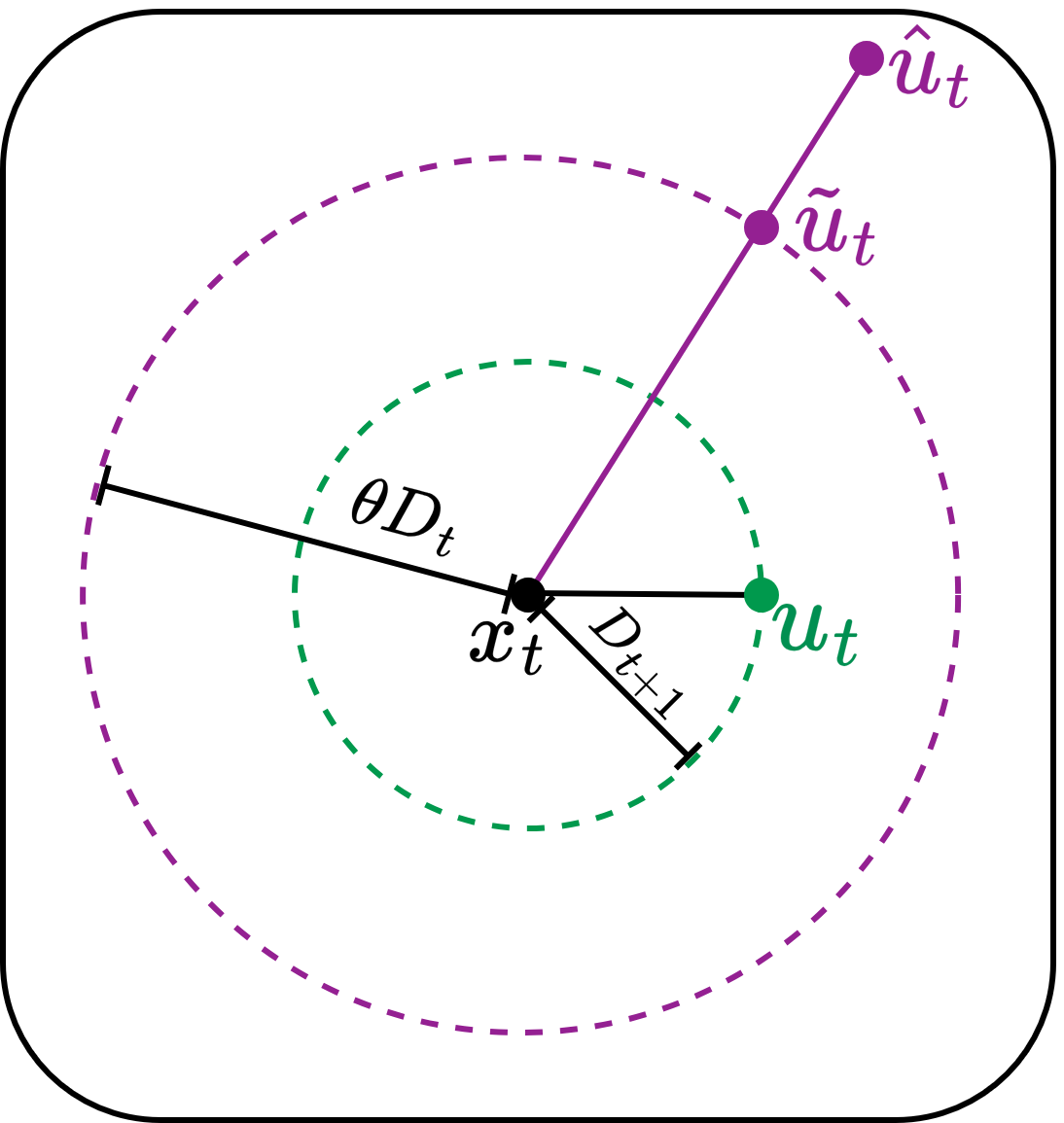}
  \captionof{figure}{Actual vs. predicted targets for a time step, and the corresponding update of $D_t$.}
  \label{fig:cort_example}
\end{minipage}

\label{fig:minipage_alg_fig}
\vspace{-2mm}
\end{figure}



  
  



\begin{theorem}
\label{thm:learning_augmented_perf}
Given parameter $\theta$, \cort is $\emph{\texttt{DF}}(\emph{\algName}, \emph{\informedGreedy}) (1 + \theta^2\mathcal{O}(1))$-robust and $\mathcal{C}$-consistent where:
\begin{align}
    \mathcal{C} \leq \psi(\theta) + (1 - \psi(\theta)) \cdot \emph{\texttt{DF}}(\emph{\algName}, \emph{\informedGreedy}) + \frac{2 \lambda_1 \lambda_2 \ell^2}{m \eta (\eta - m \lambda_1)} \cdot \frac{\theta^2}{1 + \theta^2},
\end{align}
and $\psi(\theta):\mathbb{R}_{\geq 0} \to [0,1]$ is an increasing function satisfying $\psi(0) = 0$ and $\psi(\infty) = 1$.
\end{theorem}
The full proof of Theorem~\ref{thm:learning_augmented_perf} is provided in Appendix~\S\ref{app:learning_augmented_proof}. As a sketch, we first show that the cost incurred by \cort deviates from that of \algName by at most a linear function of the aggregate deviation $\sum_t \|x_t - \tilde{x}_t\|^2$, where $x_t$ and $\tilde{x}_t$ denote the actions of \algName and \cort at time $t$, respectively. We then prove that this deviation is upper bounded by a factor proportional to $\theta^2$ times the cost of \algName, establishing the robustness guarantee. Furthermore, under perfect predictions, to construct a challenging instance, an adversary must increase the separation between its own action and that of \algName over time. Also, in such condition, the action of \cort at each time step is a convex combination of the actions of \algName and \informedGreedy. This linear relationship among the actions allows us to show that the cost of \cort is a convex combination of the costs of \algName and \informedGreedy, up to a bounded additive error, which yields the consistency bound.


\begin{remark}
\label{rem:Cort_rob_cons}
Theorem~\ref{thm:learning_augmented_perf} illustrates a trade-off in \emph{\cort} between its consistency and robustness, governed by the parameter $\theta$. As $\theta$ increases, robustness improves at most quadratically, while the consistency decreases. In the limit as $\theta \to \infty$, \emph{\cort} achieves its best possible consistency but completely sacrifices robustness.
\end{remark}

\vspace{-3mm}
\section{Case Study: Resource Allocation for Elastic and Inelastic Workloads}
\label{sec:experiments}
\vspace{-3mm}
We consider a case study involving resource allocation in cloud computing platforms handling both elastic and inelastic workloads. In this setting, we evaluate our proposed algorithms for \probName and compare them in average and adversarial scenarios.

\textbf{Experimental setup.} We model a cloud computing platform comprising multiple independent resources (e.g., processing units such as CPUs or GPUs), serving two categories of jobs. The first category, \emph{inelastic} jobs, consists of online job requests that require immediate allocation of resources, which remain occupied until the job is completed. The second category, \emph{elastic} jobs, comprises predefined jobs that can be paused and resumed over time.


The platform dynamically allocates a subset of resources to elastic workloads, while the remaining units are used to process inelastic workloads. The goal is to maintain long-term SLA requirements close to predefined targets, while serving as many inelastic jobs as possible. These inelastic workloads may vary over time (e.g., due to hourly or daily patterns), making future demand difficult to predict. At each time step, the system must decide what fraction of processing units to allocate to elastic jobs, leaving the remainder for inelastic requests.


\textbf{Constructing the \probName instance.} We construct instances of \probName as follows: the platform acts as the decision-making agent. At time $t$, the agent selects an action $x_t$, representing the fraction of available resources allocated to inelastic jobs ($\mathbb{D} = [0, 1]$). The target for the processing rate of elastic jobs is denoted by $\tau_t$, defined over a moving window of size $w$.  In addition, $1-u_t$ shows the workload demand of inelastic jobs during the next processing interval. In this setting, the \emph{tracking cost} captures deviations from the target elastic processing rate, while the \emph{adversarial cost} measures the gap between the actual allocation to elastic jobs and the maximum feasible allocation that would still satisfy all inelastic job requests.

\textbf{Workload dataset and parameter settings.} We use CPU utilization traces from the Google Cluster dataset (GCD)~\cite{GCD_dataset}, which contains 
utilization records from a total of 1,600 virtual machines. The dataset provides CPU and memory utilization measurements at five-minute intervals.
 We divide each day into three workload periods: 8\,PM--4\,AM (off-peak; low demand), 4\,AM--12\,PM (mid-peak; medium demand), and 12\,PM--8\,PM (on-peak; high demand). Accordingly, we set $\tau_t = 0.4$ during low-demand hours, $\tau_t = 0.3$ during medium-demand hours, and $\tau_t = 0.2$ during high-demand hours. These thresholds result in an average allocation of approximately $30\%$ across the day. (we have also evaluated other daily averages of $\tau_t$; results are provided in Appendix~\S\ref{app:experimental_detail}) At each time step $t$, we extract the inelastic job utilization from the GDC dataset and define $u_t$ as one minus this utilization.  To model adversarial behavior, we use a standard convex cost function, $f_t(x) = \|x\|^2$ which is commonly used in the literature~\cite{shi2020online, agrawal2019differentiable} We vary $\lambda_1$, $w$, $\theta$, $\lambda_2$, and the daily average of $\tau_t$ to evaluate their influence on the performance of online algorithms. When analyzing each parameter, we fix the others as follows: $\lambda_1 = 1$ (equal weight on elastic and inelastic jobs), $\lambda_2 = 0.1$ (to assign a 10\% weight to job-switching costs), $w = 12$ (corresponding to a one-hour history window), and $\theta = 0.5$.


\textbf{Prediction models.} Since both \cort and \greedy rely on predictions of $u_t$, we evaluate three prediction models in our analysis:  
(1) \emph{Predictor:} We employ an LSTM-based model~\cite{hochreiter1997long} to forecast $u_t$ based on its historical values (see Appendix~\S\ref{app:experimental_detail} for details).  
(2) \emph{Pessimistic:} We define the prediction as $\hat{u}_t = x_t + (x_t - u_t)$, where $x_t$ is the action taken by \algName at time $t$. This formulation reflects $u_t$ across $x_t$, resulting in a prediction that is deliberately misaligned with the true value, simulating an adversarial scenario.  
(3) \emph{Optimistic:} This model assumes perfect prediction scenario, i.e., $\hat{u}_t = u_t$.


\begin{figure}[t]
  \centering
  \begin{subfigure}[t]{0.48\linewidth}
    \centering
    \includegraphics[width=\linewidth]{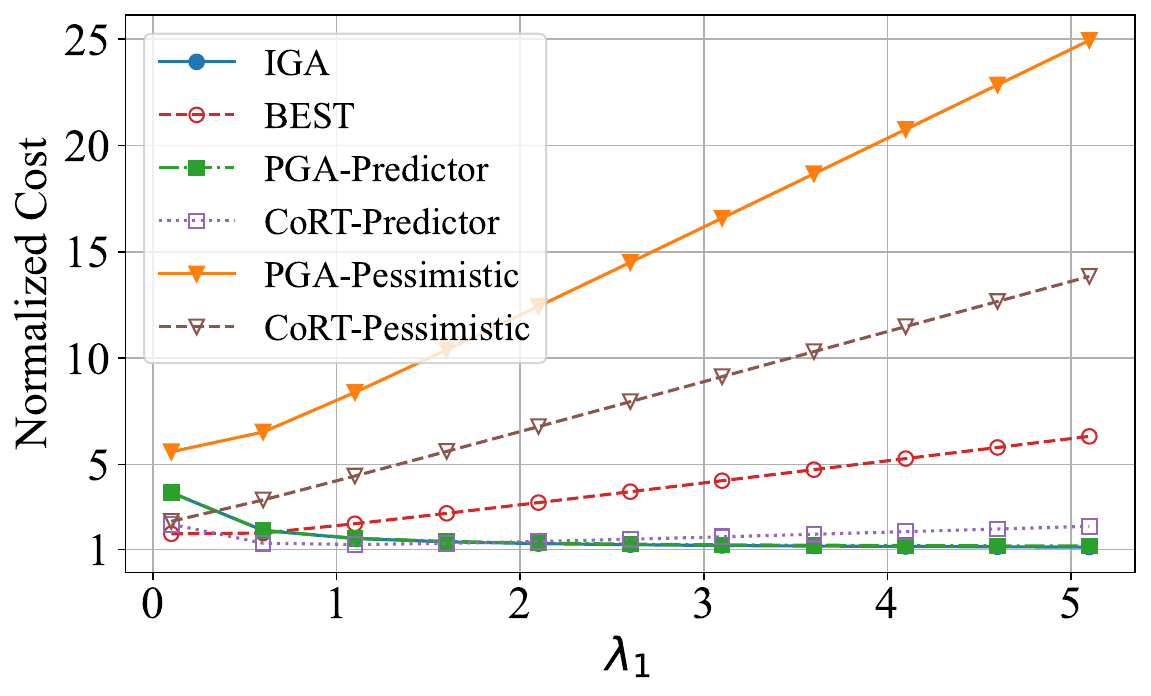}
    \caption{ Impact of parameter $\lambda_1$}
    \label{fig:elastic_l1}
  \end{subfigure}
  \hfill
  \begin{subfigure}[t]{0.48\linewidth}
    \centering
    \includegraphics[width=\linewidth]{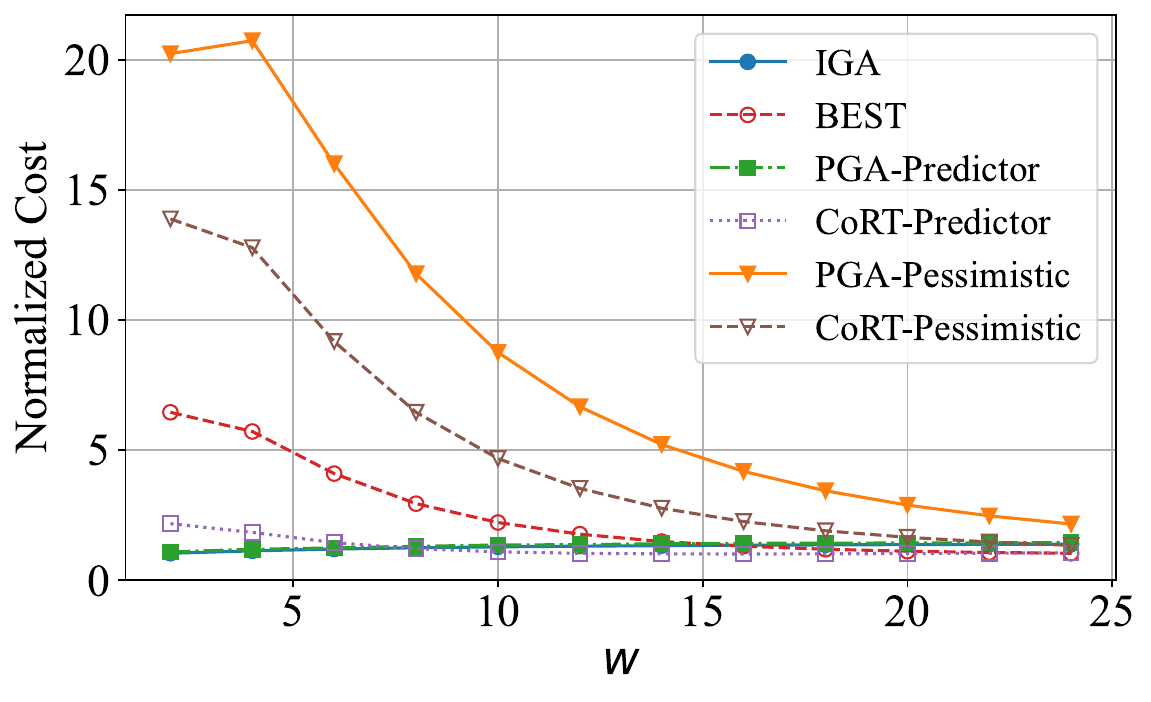}
    \caption{Impact of history's length $w$}
    \label{fig:elastic_w}
  \end{subfigure}
  \vspace{-0.5em}
    \caption{Impact of $\lambda_1$ (a) and $w$ (b) on the cost of different algorithms. Increasing $\lambda_1$ or decreasing $w$ amplifies the influence of the adversarial cost term, leading to higher overall costs.}
  \label{fig:elastic_exp}
  \vspace{-2mm}
\end{figure}

\begin{figure}[t]
    \centering
    \begin{subfigure}[t]{0.32\textwidth}
        \centering
        \includegraphics[width=\textwidth]{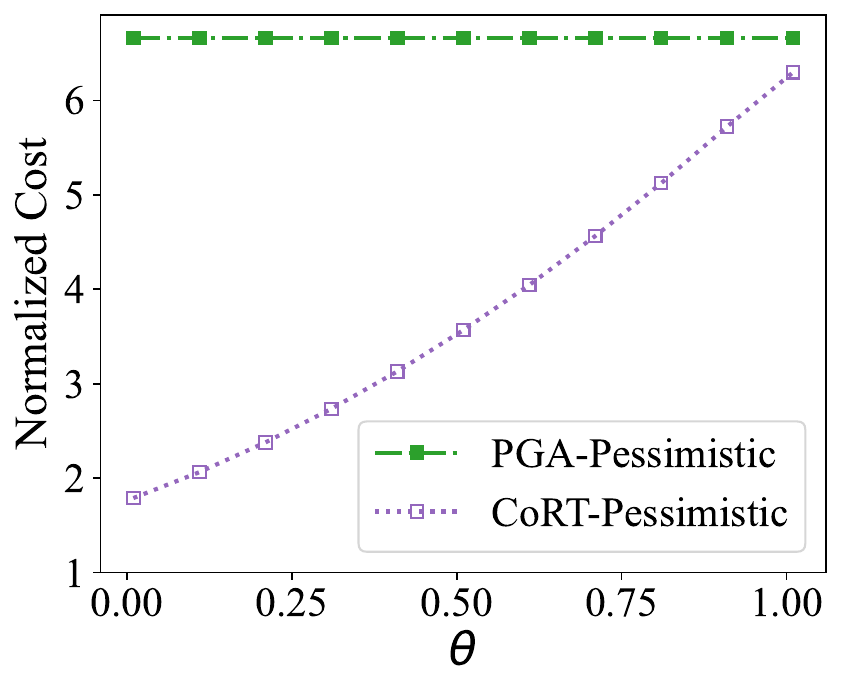}
        \label{fig:cort_pess}
    \end{subfigure}
    \hfill
    \begin{subfigure}[t]{0.32\textwidth}
        \centering
        \includegraphics[width=\textwidth]{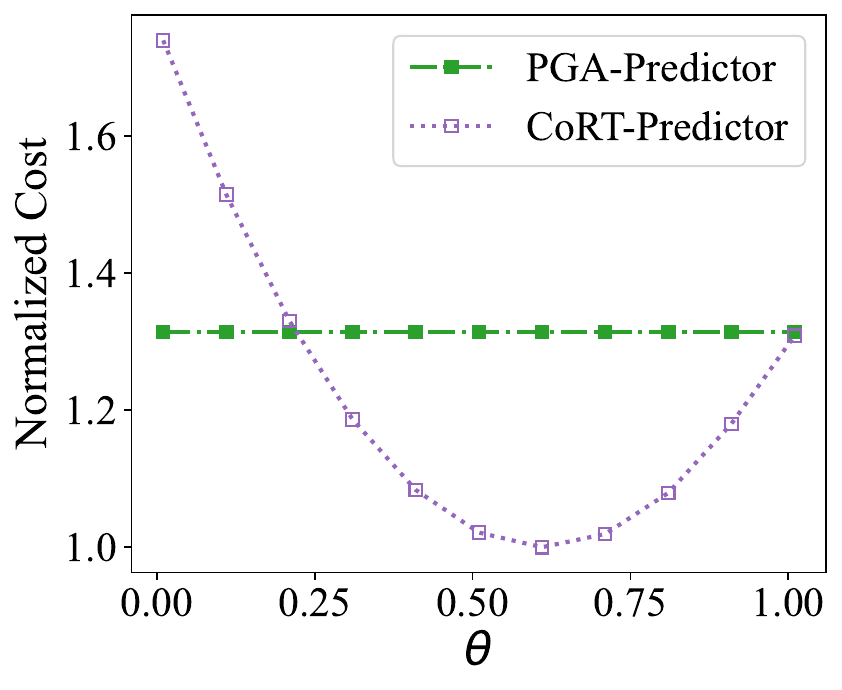}
        \label{fig:cort_pred}
    \end{subfigure}
    \hfill
    \begin{subfigure}[t]{0.32\textwidth}
        \centering
        \includegraphics[width=\textwidth]{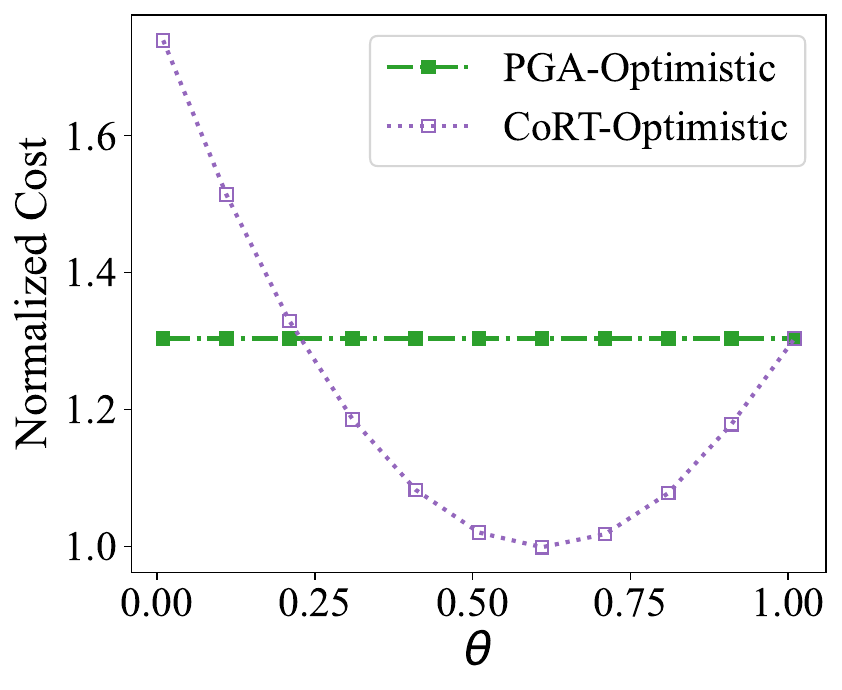}
        \label{fig:cort_opt}
    \end{subfigure}
    \vspace{-3mm}
    \caption{Comparison of the cost of \greedy and \cort as a function of $\theta$ under three prediction scenarios: pessimistic prediction (left), LSTM-based prediction (center), and perfect prediction (right). }
    \label{fig:exp_phi_impact}
    \vspace{-3mm}
\end{figure}

\textbf{Experimental results.} Figure~\ref{fig:elastic_exp} illustrates the impact of varying the parameter $\lambda_1$ and the history length $w$ on the cost of online algorithms. As shown, increasing $\lambda_1$ magnifies the influence of the adversarial cost component within the overall cost function. Consequently, the cost of online algorithms such as \algName, \greedy-Pessimistic, and \cort-Pessimistic---each lacking foresight into the adversary’s future targets---increases almost linearly with respect to $\lambda_1$, confirming the trend described in Remark~\ref{rem:best_linear_lambda}. Notably, the increase in cost for \algName is significantly smaller than that of \greedy-Pessimistic and \cort-Pessimistic, indicating its stronger robustness. Additionally, we observe that as the history length $w$ increases, the importance of action smoothness becomes more prominent, making the problem easier for online algorithms. In such settings, the gap between the cost of online algorithms and the optimal offline algorithm tends to narrow. Finally, we observe that the costs incurred by \greedy-Predictor and \cort-Predictor are close to those of \informedGreedy, which is due to the high accuracy of the \emph{Predictor} model in predicting $u_t$ (see Appendix~\S\ref{app:experimental_detail} for additional details).  

Another observation is that in certain problem instances, algorithms such as \cort-Predictor and \algName can achieve a lower cost than \informedGreedy. This may seem counterintuitive, as \informedGreedy has full knowledge of the adversary's target at the current time step. However, it still lacks information about future trajectory targets and future adversarial behavior. As a result, \algName---by leveraging history more effectively---can outperform it in some cases, but not certainly in worst-case as indicated in Theorem~\ref{thm:bestAlg_CR}.

Figure~\ref{fig:exp_phi_impact} illustrates the impact of the parameter $\theta$ on the cost of \cort under three prediction models: \emph{Optimistic}, \emph{Predictor}, and \emph{Pessimistic}. The results show that when \cort is provided with an adversarial prediction of $u_t$ (i.e., the \emph{Pessimistic} model), its cost increases almost quadratically with respect to $\theta$, confirming the theoretical result in Theorem~\ref{thm:learning_augmented_perf}. This suggests that, to ensure strong robustness, smaller values of $\theta$ should be chosen. Conversely, the analysis using the \emph{Optimistic} predictor indicates that increasing $\theta$ can lead to lower costs, thereby improving consistency. Together, these results highlight a fundamental trade-off between consistency and robustness in the performance of \cort, governed by the choice of the parameter $\theta$. We conducted further experimental analyses, the details of which are presented in Appendix~\S\ref{app:experimental_detail}.

\section{Concluding Remarks}
\label{sec:conclusion}
\vspace{-2mm}

We introduced a new framework for Smoothed Online Optimization in target tracking, which unifies tracking of a dynamic target, robustness to adversarial perturbations, and switching costs, into a single principled formulation. Our proposed algorithms, \algName and its learning-augmented counterpart \cort, offer both theoretical guarantees and strong empirical performance in applications such as elastic and inelastic workload scheduling. 
A promising direction is to design robust and competitive algorithms that relax the convexity and smoothness assumptions, thereby extending applicability to a broader range of practical settings. On the learning-augmented front, an interesting future work is to develop risk-aware learning-augmented algorithms that can dynamically adjust their reliance on predictions based on uncertainty quantification models~\cite{sun2024online,christianson2024risk}.


\section{Acknowledgment}
This research was supported in part by NSF grants CAREER 2045641, CPS-2136199, CNS-2106299, CNS-2102963, CSR-1763617, CNS-2106463, and CNS-1901137. We acknowledge their financial assistance in making this project possible.

\bibliographystyle{plain}
\bibliography{references} 
\newpage
\appendix

\section{Additional Literature Review}
\label{sec:background}

\textbf{Online Linear Tracking Control Problem} The online linear tracking control problem~\cite{lin2021perturbation,lin2024online,zhang2022adversarial} models a sequential decision-making scenario in which an agent selects actions over a horizon of $T$ time steps. At each time step $t$, given the current state $s_t \in \mathbb{R}$, the agent selects an action $x_t \in \mathbb{R}^d$. The environment then updates the state $s_{t+1}$ based on a known dynamics model that incorporates the previous state $s_t$, the current action $x_t$, and potentially an adversarial perturbation. The agent incurs a cost composed of a state-dependent loss $f_t(s_{t+1})$ and an action-dependent loss $c_t(x_t)$. 

This framework introduces significant challenges due to the interaction between the agent’s actions, the system dynamics, and adversarial perturbations—making it difficult to match the performance of an optimal offline algorithm that knows all future perturbations in advance. Nonetheless, the problem is highly relevant in several practical domains. For example, in \textit{autonomous systems}~\cite{chen2022online,xu2024adaptive}, self-driving vehicles must adjust their control strategies to track target trajectories despite disturbances such as wind or road condition changes. In \textit{energy grid management}~\cite{yan2023onlineGrid,zhang2023online}, EV charging infrastructures must dynamically adapt to unpredictable demand fluctuations while maintaining load balance. The model is also applicable to \textit{network congestion control}~\cite{abbasloo2020classic,li2021auto}, where network traffic must be regulated under fluctuating bandwidth constraints, and to \textit{robotic manipulation}~\cite{liu2020multi,amersdorfer2020real}, where robots must precisely follow motion plans despite external forces.

A common assumptions in previous works is the convexity of cost functions. If $x_t^*$ denotes the minimizer of the per-step cost, then tracking $x_t^*$ closely over time is key to minimizing cumulative cost. In some variants, such as the \textit{online tracking control with memory}, cost functions additionally depend on a history window of past actions, typically of size $w$. That is, the action cost at time $t$ may be a function $c_t(x_t, x_{t-1}, \ldots, x_{t-w})$. A prominent special case is the \textit{switching cost model}, where the cost penalizes rapid changes between consecutive actions. This is often expressed as
and has been widely studied~\cite{zhang2022adversarial,zhang2022optimal,zhao2023non,zhang2022smoothed,lechowicz2024online} to encourage smoother control policies.

Theoretical guarantees for this problem have been the subject of extensive research. In~\cite{zhang2022adversarial}, the authors propose an algorithm for online tracking control with memory using online convex optimization techniques and establish a regret bound of $\mathcal{O}(\log T \cdot \sqrt{T})$. In~\cite{lin2021perturbation}, a predictive control algorithm is introduced that forecasts $k$ steps ahead and selects actions accordingly. They show that the algorithm achieves linear regret in $T$, with the regret decreasing exponentially as a function of the prediction window size $k$. However, they also observe that the competitive ratio can increase exponentially with $k$, revealing a trade-off: longer prediction windows may reduce regret due to foresight, but at the cost of higher sensitivity to prediction errors. More recently,~\cite{lin2024online} proposed a gradient-based method that achieves a sublinear regret of $\mathcal{O}(\sqrt{T})$.

However, prior theoretically grounded works in this area primarily focus on minimizing short-term state and action costs, often under linear dynamic assumptions and immediate tracking objectives, without explicitly accounting for long-term behavioral constraints. As a result, they do not capture our smoothed tracking objective, which requires the agent to keep the average of its actions over a window close to a dynamically evolving sequence of targets.

\textbf{Online Convex Optimization}
The convexity assumption and leveraging an online convex optimization techniques is common in design and analysis of algorithms for online optimization for target tracking~\cite{zhao2023non,kumar2024online,zhao2024adaptivity,adib2023online,mhaisen2024adaptive,yan2023online}. In classic online convex optimization, the agent must selects an action sequentially in order to minimize the aggregate time dependent cost function. Different versions of online convex optimization have been introduced and studied in the literature. This includes time dependent convex cost function, $c_t(x_t)$~\cite{hazan2007logarithmic,jenatton2016adaptive,guo2022online}, convex optimization with switching cost~\cite{zhao2020dynamic,yu2017online,lin2012online}, convex optimization with memory~\cite{anava2015online,shi2020online}, enhancement using prediction~\cite{chen2015online,li2019online,li2020leveraging}, and considering adversarial perturbation in the cost function~\cite{shi2020online,foster2020logarithmic,cutkosky2017stochastic}. The similarity between online convex optimization and online optimization for target tracking problem and numerous number of previous works in online convex optimization have helped researchers to use their result for solving different aspects of the online target tracking problem. However, most of these works either omit the notion of tracking a time-varying target or focus only on instantaneous objectives, without modeling the long-term smoothed tracking similar to our targeted problem setting.

\section{Proofs of Theoretical Result}
\label{app:proofs}

We start by providing proofs of key lemmas that support the theoretical results presented in the main body of the paper.

\begin{proposition}[Lemma 4 from~\cite{shi2020online}]
\label{prop:prev_work_smoothness} 
If $f: \mathbb{R}^d \rightarrow \mathbb{R}^{+}\cup\{0\}$ is convex and $\ell$-smooth, for any input point $x,$ and $y$, and positive variable $\delta$ we have:
\begin{align*}
    f(y) \leq (1+\delta)f(x) + (1+\frac{1}{\delta})\ \frac{\ell}{2} ||y-x||^2.
\end{align*}
\end{proposition}
\begin{proof}
    The proof of the above proposition is given in lemma~4 of~\cite{shi2020online}.
\end{proof}

\begin{lemma}
\label{lem:sum_diff_ys}
Consider the action selection algorithm defined as:
\begin{align*}
    x(u, h) = \argmin_x\ ||\frac{x+h}{w+1} - \tau||^2+ \lambda_1 f(x - u) + + {\lambda_2} ||x - z||^2,
\end{align*}
where $f(\cdot)$ is an $m$-strongly convex, and $\ell$-smooth function. Then, the following inequality holds:

\begin{align*}
    ||x(\hat{u}, \hat{h}) - x(u, h)|| \;\leq\;&  \frac{1}{\eta}\bigg[\lambda_1\ell||\hat{u} - u|| +  \frac{1}{(w+1)^2}||\hat{h} - {h}||\bigg].
\end{align*}
where $\eta = \frac{2}{(w+1)^2} + \lambda_1 m + 2\lambda_2$.
\end{lemma}
\begin{proof}
Let define function $\phi\bigl(x; u, h\bigr)$ as follows:
\[
\phi\bigl(x; u, h\bigr) \;=\;
||\tfrac{x+h}{w+1} - \tau||^2 
+ \lambda_1\,f\bigl(x - u\bigr)
+{\lambda_2}\,\|x - z\|^2.
\]

We can rewrite it as:
\begin{align*}
    \phi\bigl(x; u, h\bigr) \;=\;
\frac{1}{(w+1)^2}\,\|x + h - (w+1)\tau\|^2 
+ \lambda_1\,f\bigl(x - u\bigr)
+ {\lambda_2}\,\|x - z\|^2.
\end{align*}

The gradient of $\phi\bigl(x; u, h\bigr)$ can be derived as follows:
\begin{align*}
    \nabla_x\,\phi\bigl(x;u,h\bigr) \;=\; 
\frac{2}{(w+1)^2}\,\bigl[x + h - (w+1)\tau\bigr] +
\lambda_1\,\nabla f\bigl(x - u\bigr) +
2 \lambda_2\,\bigl(x - z\bigr).
\end{align*}

By definition, we have:
\begin{align*}
    x(u,h) &= \arg\min_{x}\;\phi(x;u,h),\\
    x(\hat{u},\hat{h}) &= \arg\min_{x}\;\phi(x;\hat{u},\hat{h}),
\end{align*}
which implies
\begin{align}
\label{eq:zero_grad_x}
    \nabla_x\,\phi\bigl(x(u,h);u,h\bigr) &= 0,\\
    \label{eq:zero_grad_xhat}
    \nabla_x\,\phi\bigl(x(\hat{u},\hat{h});\hat{u},\hat{h}\bigr) &= 0.
\end{align}

According to~\eqref{eq:zero_grad_xhat} we get:
\begin{align}
\nabla_x\,\phi\bigl(x(\hat{u},\hat{h});u,h\bigr)
\;=\;&
\nabla_x\,\phi\bigl(x(\hat{u},\hat{h});u,h\bigr) -
\nabla_x\,\phi\bigl(x(\hat{u},\hat{h});\hat{u},\hat{h}\bigr)\notag\\
\label{eq:grad_x_diff}
\;=\;& \frac{2}{(w+1)^2}\,(\hat{h}-h) +
\lambda_1\,\Bigl[
    \nabla f\bigl(x(\hat{u},\hat{h}) - u\bigr) -
    \nabla f\bigl(x(\hat{u},\hat{h}) - \hat{u}\bigr)
\Bigr].
\end{align}

Since $f(\cdot)$ is $\ell$-strongly smooth, we get:
\begin{align}
    \label{eq:f_grad_smooth}
    ||\nabla f\bigl(x(\hat{u},\hat{h}) - u\bigr) -
    \nabla f\bigl(x(\hat{u},\hat{h}) - \hat{u}\bigr)|| \leq&\ \ell\ ||\big(x(\hat{u},\hat{h}) - u\big) - \big(x(\hat{u},\hat{h}) - \hat{u}\big)|| \leq\ \ell\ ||\hat{u} - u||. 
\end{align}

In addition , since $f(\cdot)$ is $\eta$-strongly convex, we get:
\begin{align}
    \eta\ ||x(\hat{u}, \hat{h}) - x(u,h)|| \leq&\ ||\nabla_x \phi(x(\hat{u},\hat{h});u,h) - \nabla_x \phi(x(u,h);u,h)||\notag\\
    \label{eq:f_grad_convex}
    \leq&\ ||\nabla_x \phi(x(\hat{u},\hat{h});u,h)||.
\end{align}
where the last inequality holds due to~\eqref{eq:zero_grad_x}. Combining~\eqref{eq:grad_x_diff},~\eqref{eq:f_grad_smooth}, and~\eqref{eq:f_grad_convex} completes the proof.
\end{proof}

\begin{lemma}
    \label{lem:phi_chi_convexity}
     Consider the function $g_t(u)$ defined as:
    \begin{align*}
    g_t(u) =& \min_x ||\frac{x + h_t}{w+1} - \tau_t||^2 + {\lambda_1} \cdot f_t(x - u) +  {\lambda_2} ||x - x_{t-1} ||^2,
\end{align*}

where $f_t(.)$ is $m$-strongly convex function. The function $g_t(u)$ is $\eta_2$-strongly convex, with $\eta_2$ given by:
\begin{align*}
    \eta_2 = \ m\lambda_1(1-\frac{m \lambda_1}{\eta}),
\end{align*}
where $\eta = \frac{2}{(w+1)^2} + m\cdot \lambda_1 + 2\lambda_2$.
\end{lemma}
\begin{proof}
   To simplify the analysis, we rewrite $g_t(u)$ as:
    \begin{align*}
    g_t(u) =& \min_x ||\frac{x + u + h_t}{w+1} - \tau_t||^2 + {\lambda_1} \cdot f_t(x) +  {\lambda_2} ||x + u - x_{t-1} ||^2,
\end{align*}

To prove that \( g_t(u) \) is \( \eta_2 \)-strongly convex, we need to verify the following inequality for any \( u_1, u_2 \) and \( \gamma \in [0, 1] \):
\begin{align*}
    g_t(\gamma u_1 + (1-\gamma) u_2) \leq \gamma g_t(u_1) + (1-\gamma) \phi(u_2) -\frac{\eta_2}{2} \gamma (1-\gamma) ||u_1 - u_2||^2.
\end{align*}

Let: 
\begin{align*}
    x_1 &= \argmin_x  \left\| \frac{x + u_1 + h_t}{w+1} - \tau_t \right\|^2 + \lambda_1 f_t(x) + {\lambda_2} \|x + u_1 - x_{t-1}\|^2, \\
    x_2 &= \argmin_x  \left\| \frac{x + u_2 + h_t}{w+1} - \tau_t \right\|^2 + \lambda_1 f_t(x) + {\lambda_2} \|x + u_2 - x_{t-1}\|^2.
\end{align*}

As $g_t(\cdot)$ is strongly convex we get:
\begin{align*}
    &\gamma g_t(u_1) + (1-\gamma) g_t(u_2) - \frac{\eta_2}{2} \gamma (1-\gamma) ||u_1 - u_2||^2\\
    =& {\gamma }||\frac{x_1 + u_1 + h_t}{w+1} - \tau_t||^2 + \gamma {\lambda_1} \cdot f_t(x_1) +  \gamma \lambda_2 ||x_1 + u_1 - x_{t-1} ||^2\\
    +&  {(1-\gamma)}||\frac{x_2 + u_2 + h_t}{w+1} - \tau_t||^2 + (1-\gamma) {\lambda_1} \cdot f_t(x_2) +  (1-\gamma) \lambda_2 ||x_2 + u_2 - x_{t-1} ||^2\\
    -& \frac{\eta_2}{2} \gamma (1-\gamma)  ||u_1 - u_2||^2\\
    \geq& {\lambda_1} \cdot f_t(\gamma x_1 + (1-\gamma) x_2) + \frac{m \cdot \lambda_1}{2} \gamma(1-\gamma) ||x_1 - x_2||^2 + {\gamma}||\frac{x_1 + u_1 + h_t}{w+1} - \tau_t||^2 \\
    +& {(1-\gamma)}||\frac{x_2 + u_2 + h_t}{w+1} - \tau_t||^2 + \gamma \lambda_2 ||x_1 + u_1 - x_{t-1}||^2 + (1-\gamma) \lambda_2 ||x_2 + u_2 - x_{t-1}||^2\\
    -& \frac{\eta_2}{2} \gamma (1-\gamma) ||u_1 - u_2||^2,
    \end{align*}
    where the above inequality holds since $f_t(.)$ is $m$-strongly convex. By using the definition of $g_t(.)$ we get:
    \begin{align*}
    &\gamma g_t(u_1) + (1-\gamma) g_t(u_2) - \frac{\eta_2}{2} \gamma (1-\gamma) ||u_1 - u_2||^2\\
    \geq&  g_t(\gamma u_1 + (1-\gamma) u_2) -  ||\frac{\gamma (x_1 + u_1) + (1-\gamma) (x_2 + u_2) + h_t}{w+1} - \tau_t||^2\\
    -& \lambda_2 ||\gamma (x_1 + u_1) + (1-\gamma)(x_2 + y_2) - x_{t-1}||^2 + \frac{m \cdot \lambda_1}{2} \gamma (1-\gamma) ||x_1 - x_2||^2\\
    +&  {\gamma}||\frac{x_1 + u_1 + h_t}{w+1} - \tau_t||^2  + {(1-\gamma)}||\frac{x_2 + u_2 + h_t}{w+1} - \tau_t||^2\\
     +&  \gamma \lambda_2 ||x_1 + u_1 - x_{t-1}||^2 + (1-\gamma) {\lambda_2} ||x_2 + u_2 - x_{t-1}||^2 - \frac{\eta_2}{2} \gamma (1-\gamma)  ||u_1 - u_2||^2,\\
     \end{align*}

    Now using the fact that  $\frac{1}{2}||\sqrt{{z_1}}{x} - z_2||^2$ is ${z_1}$-strongly convex we get: 
     \begin{align}
     &\gamma g_t(u_1) + (1-\gamma) g_t(u_2) - \frac{\eta_2}{2} \gamma (1-\gamma) ||u_1 - u_2||^2\\
    \geq& g_t(\gamma u_1 + (1-\gamma) u_2) -   ||\frac{\gamma (x_1 + u_1) + (1-\gamma) (x_2 + u_2) + h_t}{w+1} - \tau_t||^2\notag\\
    -& \lambda_2 ||\gamma (x_1 + u_1) + (1-\gamma)(x_2 + y_2) - x_{t-1}||^2\notag\\
     +&  ||\frac{\gamma (x_1 + u_1) + (1-\gamma) (x_2 + u_2) + h_{t} }{w+1} - \tau_t||^2\notag \\
     +& \frac{1}{(w+1)^2}\gamma (1-\gamma) ||(x_1 -x_2) + (u_1 - u_2)||^2 \notag\\
     +& \lambda_2 ||\gamma(x_1 + u_1) + (1-\gamma) (x_2 + u_2) - x_{t-1}||^2 + \lambda_2 \gamma(1-\gamma) ||x_1 - x_2 + u_1 - u_2||^2\notag\\
    -& \frac{\eta_2}{2} \gamma (1-\gamma) ||u_1 - u_2||^2 + \frac{m \cdot \lambda_1}{2}\gamma(1-\gamma)||x_1 - x_2||^2\notag\\
    =& g_t(\gamma u_1 + (1-\gamma) u_2) + \frac{m \cdot \lambda_1}{2} \gamma (1-\gamma) ||x_1 - x_2||^2 \notag\\
    +& \frac{1}{(w+1)^2}\gamma (1-\gamma) ||(x_1 -x_2) + (u_1 - u_2)||^2 + \lambda_2 \gamma(1-\gamma) ||x_1 - x_2 + u_1 - u_2||^2 \notag\\
    \label{eq:convexity_of_phi_1}
    -& \frac{\eta_2}{2} \gamma (1-\gamma) ||u_1 - u_2||^2.
\end{align}

In addition, we have:
\begin{align}
    &{m \cdot \lambda_1} ||x_1 - x_2||^2 + \frac{2}{(w+1)^2} ||(x_1 -x_2) + (u_1 - u_2) ||^2\notag\\
    +& {2\lambda_2} ||x_1 - x_2 + u_1 - u_2||^2 - \eta_2 ||u_1 - u_2||^2\notag\\
    \geq& (m \cdot \lambda_1 + \frac{2}{(w+1)^2} + 2\lambda_2 ) ||x_1 - x_2||^2 + (\frac{2}{(w+1)^2} + 2\lambda_2  - \eta_2) ||u_1 - u_2||^2\notag\\
    +& 2(\frac{2}{(w+1)^2} + 2\lambda_2 ) (x_1 - x_2)\cdot(u_1 - u_2)\notag\\
    \label{eq:convexity_of_phi_2}
    =& \bigg( \sqrt{\eta} (x_1 - x_2) + \frac{\eta - m \cdot \lambda_1}{\sqrt{\eta}} (u_1 - u_2) \bigg)^2 \geq 0.
\end{align}

Finally inserting Equation~\eqref{eq:convexity_of_phi_2} into~\eqref{eq:convexity_of_phi_1} completes the proof.

\end{proof}

\begin{lemma}[Adaptation of the Cauchy–Schwarz Bound]
\label{lem:sum_chi}
Consider two sequences of actions \( x_{1:T} := [x_1, x_2, ..., x_T] \) and \( y_{1:T} := [y_1, y_2, ..., y_T] \). The following inequality always holds:
\[
\sum_{t=1}^T \left\|\sum_{i=1}^w (y_{t-i} - x_{t-i})\right\|^2 \leq w^2 \sum_{t=1}^T \|y_t - x_t\|^2.
\]
\end{lemma}
\begin{proof}
Expanding the left-hand side:
\[
\sum_{t=1}^T \left\|\sum_{i=1}^w (y_{t-i} - x_{t-i})\right\|^2 = \sum_{t=1}^T w^2 \left\|\sum_{i=1}^w \frac{1}{w}(y_{t-i} - x_{t-i})\right\|^2.
\]
Applying Jensen's inequality to the inner sum, we have:
\[
\left\|\sum_{i=1}^w \frac{1}{w}(y_{t-i} - x_{t-i})\right\|^2 \leq \sum_{i=1}^w \frac{1}{w} \|y_{t-i} - x_{t-i}\|^2.
\]
Substituting this into the original expression:
\[
\sum_{t=1}^T \left\|\sum_{i=1}^w (y_{t-i} - x_{t-i})\right\|^2 \leq \sum_{t=1}^T w^2 \sum_{i=1}^w \frac{1}{w} \|y_{t-i} - x_{t-i}\|^2.
\]
Reorganizing the terms:
\[
\sum_{t=1}^T \left\|\sum_{i=1}^w (y_{t-i} - x_{t-i})\right\|^2 \leq w^2 \sum_{t=1}^T \|y_t - x_t\|^2,
\]

This completes the proof.
\end{proof}

\subsection{Proof of Theorem~\ref{thm:CR_perfectPrediction}}
\label{app:CR_perfectPrediction_proof}
\begin{proof}
    Define $\eta = 2 / (w+1)^2 + m\lambda_1 + 2\lambda_2$ and the function $\mathcal{F}_1(t)$ as:
   
    \begin{equation*}
        \mathcal{F}_1(t) = \frac{\eta}{2} ||x_t - x^*_t||^2.
    \end{equation*}
    where $x^*_t$ represents the action of the optimal offline algorithm at time step $t$. Summing $\mathcal{F}_1(t)$ over all time steps gives:
    
    \begin{align*}
        \sum_{t=1}^T \mathcal{F}_1(t) =& \sum_{t=1}^T \frac{\eta}{2} ||x_t - x^*_t||^2\\
        = & \sum_{t=1}^T \mathcal{F}_1(t-1) + \frac{\eta}{2}  \bigg( ||x_T - x^*_T||^2 - ||x_{0} - x^*_{0}||^2\bigg) = \mathcal{F}_1(T) + \sum_{t=1}^T \mathcal{F}_1(t-1)
    \end{align*}
    which yields:
    \begin{equation}
    \label{eq:perfectPred_positive_g}
        \Rightarrow \sum_{t=1}^T \mathcal{F}_1(t) - \mathcal{F}_1(t-1) = \mathcal{F}_1(T) \geq 0.
    \end{equation}

Here, we used the fact that $x_0 = x^*_0$ from Assumption~\ref{asmp:start_history}. Since $||\frac{x + h_{t}}{w+1} - \tau_{t}||^2 + {\lambda_1} \ f_{t}(x - u_{t}) + \lambda_2 ||x - x_{t-1} ||^2$ is $\eta$-strongly convex with respect to $x$, and $x_t$ is the minimizer, for $w > 0$ we obtain:
    
    \begin{align}
        &\ ||\frac{x_t + h_{t}}{w+1} - \tau_{t}||^2 + {\lambda_1} \ f_{t}(x_t - u_{t}) +  \lambda_2 ||x_t - x_{t-1} ||^2\notag\\
        +\ & \frac{\eta}{2} ||x_t - x^*_t||^2 - \frac{\eta }{2} ||x_{t-1} - x^*_{t-1}||^2 \notag\\
        \leq\ & ||\frac{x^*_t + h_{t}}{w+1} - \tau_{t}||^2 + {\lambda_1} \ f_{t}(x^*_t - u_{t}) +  \lambda_2 ||x^*_t - x_{t-1} ||^2 - \frac{\eta}{2} ||x_{t-1} - x^*_{t-1}||^2\notag\\
        \label{eq:perfectPred_compare1}
        =\ & \bigg(\lambda_1 f_t(x^*_t - u_t)\bigg) + \bigg(||\frac{x^*_t + h_{t}}{w+1} - \tau_{t}||^2  + \lambda_2 ||x^*_t - x_{t-1}||^2 - \frac{\eta}{2} ||x_{t-1} - x^*_{t-1}||^2\bigg),
    \end{align}

For any positive constants $\alpha$ and $\beta$, the latter term is bounded as follows:

\begin{align}
    \ &||\frac{x^*_t + h_{t}}{w+1} - \tau_{t}||^2  + \lambda_2 ||x^*_t - x_{t-1}||^2 - \frac{\eta}{2} ||x_{t-1} - x^*_{t-1}||^2\notag\\
    \leq\ & ||\frac{x^*_t + h^*_{t}}{w+1} - \tau_{t}||^2 + ||\frac{h_t - h_t^*}{w+1}||^2 + 2 ||\frac{x^*_t + h_t^{*}}{w+1} - \tau_t ||\cdot ||\frac{h_t - h_t^*}{w+1}|| \notag\\\
    +\ &   \lambda_2 ||x^*_t - x^*_{t-1}||^2 + 2 \lambda_2 ||x^*_{t} - x^*_{t-1}||\cdot||x_{t-1}- x^*_{t-1}|| + \lambda_2 ||x_{t-1} - x^*_{t-1}||^2 \notag\\\
    -\ & \frac{\eta}{2} ||x_{t-1} - x^*_{t-1}||^2\notag\ \\
    \overset{(a)}{\leq}& ||\frac{x^*_t + h^*_{t}}{w+1} - \tau_{t}||^2 + ||\frac{h_t - h_t^*}{w+1}||^2 + \frac{1}{\beta} ||\frac{x^*_t + h_t^{*}}{w+1} - \tau_t ||^2\notag\  \\
    +\ & {\beta} ||\frac{h_t - h_t^*}{w+1}|| ^2 + \lambda_2 ||x^*_t - x^*_{t-1}||^2 + \frac{\lambda_2^2}{\alpha} ||x^*_t - x^*_{t-1}||^2\notag\\\
    +\ & {\alpha} ||x_{t-1} - x^*_{t-1}||^2 + (\frac{2\lambda_2 -\eta }{2}) ||x_{t-1} - x^*_{t-1}||^2\notag\\\
    \leq\ &  (1+\frac{1}{\beta})||\frac{x^*_t + h^*_{t}}{w+1} - \tau_{t}||^2 + ({1 + \beta}) ||\frac{h_t - h_t^*}{w+1}||^2 \notag\\\
    +\ & {\lambda_2}(1+\frac{\lambda_2}{\alpha}) ||x^*_t - x^*_{t-1}||^2 + (\frac{2\alpha + 2\lambda_2 - \eta }{2}) ||x_{t-1} - x^*_{t-1}||^2,\notag\
\end{align}
where $(a)$ follows from the AM-GM inequality. By summing over all time steps, we have:
\begin{align}
    \ &\sum_{t=1}^T \bigg( ||\frac{x^*_t + h_{t}}{w+1} - \tau_{t}||^2  + {\lambda_2} ||x^*_t - x_{t-1}||^2 - \frac{\eta}{2} ||x_{t-1} - x^*_{t-1}||^2\notag \bigg) \notag\\
    \overset{(b)}{\leq}& (1+\frac{1}{\beta}) \bigg( \sum_{t=1}^T ||\frac{x^*_t + h^*_{t}}{w+1} - \tau_{t}||^2 \bigg) + \frac{w^2(1 + \beta)}{(w+1)^2} \bigg( \sum_{t=1}^T ||x_{t} - x^*_{t}||^2 \bigg) \notag \\
    +\ &  {\lambda_2}(1+\frac{\lambda_2}{\alpha}) \bigg( \sum_{t=1}^T ||x^*_t - x^*_{t-1}||^2 \bigg) + (\frac{2\alpha + 2\lambda_2 - \eta}{2}) \bigg( \sum_{t=1}^T ||x_{t-1} - x^*_{t-1}||^2 \bigg)\notag\\
    \leq\ &  (1+\frac{1}{\beta}) \bigg( \sum_{t=1}^T ||\frac{x^*_t + h^*_{t}}{w+1} - \tau_{t}||^2 \bigg) + {\lambda_2}(1+\frac{\lambda_2}{\alpha}) \bigg( \sum_{t=1}^T ||x^*_t - x^*_{t-1}||^2 \bigg)  \notag \\
    \label{eq:perfectPred_amgm_1}
    +\ &  \bigg(\frac{2\alpha  + 2\lambda_2  + 2(1+\beta) w^2 / (w+1)^2 - \eta}{2}  \bigg) \bigg( \sum_{t=1}^T ||x_{t} - x^*_{t}||^2 \bigg) -  (\frac{2\alpha  + 2\lambda_2 - \eta}{\eta}) \mathcal{F}_1(T),
\end{align}
where $(b)$ uses Lemma \ref{lem:sum_chi}. Substituting this into \eqref{eq:perfectPred_compare1}, we obtain:

\begin{align}
    \ &\sum_{t=1}^T ||\frac{x_t + h_{t}}{w+1} - \tau_{t}||^2 + {\lambda_1} \ f_{t}(x_t - u_{t}) +  {\lambda_2} ||x_t - x_{t-1} ||^2  \notag\\
    \leq \ &  \bigg(\sum_{t=1}^T \lambda_1 f_t(x^*_t - u_t) \bigg) -  \frac{2\alpha  + 2\lambda_2 + \eta - \eta}{\eta} \mathcal{F}_1(T)\notag\\
    +\ &   (1+\frac{1}{\beta}) \bigg( \sum_{t=1}^T ||\frac{x^*_t + h^*_{t}}{w+1} - \tau_{t}||^2 \bigg) + \lambda_2(1+\frac{\lambda_2}{\alpha}) \bigg( \sum_{t=1}^T ||x^*_t - x^*_{t-1}||^2 \bigg) \notag\\
    +\ &  \bigg(\frac{2\alpha + 2\lambda_2 + 2(1+\beta)w^2/(w+1)^2 - \eta }{2} \bigg) \bigg( \sum_{t=1}^T ||x_{t} - x^*_{t}||^2 \bigg) \notag\\
    \leq\ & \max\{1+ \frac{1}{\beta},  1 + \frac{\lambda_2}{\alpha}  \}\ \text{Cost}(OPT, \mathcal{I}) \notag\\
        \label{eq:perfectPred_gen_CR}
    -\ &  \frac{2\alpha  + 2\lambda_2 }{\eta} \mathcal{F}_1(T) + \bigg(\frac{2\alpha + 2\lambda_2  + 2(1+\beta)w^2/(w+1)^2 - \eta}{2} \bigg) \bigg( \sum_{t=1}^T ||x_{t} - x^*_{t}||^2 \bigg).
\end{align}

The additive terms would be non-positive if the following inequality holds:
\begin{equation}
    \label{eq:perfectPred_alpha_range_2}
    2\alpha + 2\lambda_2  + 2(1+\beta)w^2/(w+1)^2 - \eta  \leq 0.
\end{equation}

This implies that if condition $2w^2/(w+1)^2 < m \lambda_1 + 2/(w+1)^2$ holds, the competitive ratio of the adversarial aware algorithm is upper bounded by:



\begin{equation}
    \texttt{CR}(\informedGreedy ) \leq 1 + \frac{2(\lambda_2\ (w+1)^2 + w^2)}{m \lambda_1 (w+1)^2 - 2(w^2-1)}.
\end{equation}

\end{proof}

\subsection{Proof of Theorem~\ref{thm:bestAlg_CR}}
\label{app:bestAlg_CR_proof}
\begin{proof}
We know the adversarial cost function $f_t(.)$ is $m$-strongly convex. The cost function at time step $t$, $\text{Cost}_t(x_t, h_t) =||\frac{x_t + h_t}{w+1} - \tau_t||^2 + {\lambda_1} f_t(x_t - u_t) +  \lambda_2 ||x_t - x_{t-1} ||^2$ is  $\eta$-strongly convex where $\eta$ can be calculated as follows:
\begin{align*}
    \eta =  \frac{2}{(w+1)^2} + m\cdot \lambda_1 + 2\lambda_2.
\end{align*}


Consider the following function:

\begin{align*}
    g_t(u) =& \min_x ||\frac{x_t + \hat{h}_t}{w+1} - \tau_t||^2 + {\lambda_1} f_t(x_t - u) +  \lambda_2 ||x_t - \hat{x}_{t-1} ||^2.
\end{align*}

By the process of selecting $x_t$ by \algName and the fact that function $f_t(.)$ in minimized at the origin, we reach that $u = x_t$ is the minimizer of the $g_t(u)$. From Lemma~\ref{lem:phi_chi_convexity}, $g_t(u)$ is $\eta_2 = m\lambda_1 (1-\frac{m\cdot \lambda_1}{\eta})$-strongly convex. So by the strong convexity of $g_t(.)$ we get:

\begin{align}
    &\ ||\frac{x_t + \hat{h}_t}{w+1} - \tau_t||^2 + \lambda_1 f_t(x_t - x_t) + \lambda_2 ||x_t - \hat{x}_{t-1}||^2 + \frac{\eta_2}{2} ||x_t - u_t||^2 \notag\\
    \label{eq:noisy_convex_phi}
    \leq\ & ||\frac{\hat{x}_t + \hat{h}_t}{w+1} - \tau_t||^2 + \lambda_1 f_t(\hat{x}_t - u_t) + \lambda_2 ||\hat{x}_t - \hat{x}_{t-1}||^2.
\end{align}


Also the function $\mathcal{F}_2(h) = ||\frac{x_t + h}{w+1} - \tau||^2$ is $\frac{2}{(w+1)^2}$-strongly smooth, so for any $0 < \delta_0$ we have:

\begin{align}
    \label{eq:noisy_smoothness_2}
    \frac{1}{(1+\delta_0)} ||\frac{x_t + h_t}{w+1} - \tau_t||^2 \leq  ||\frac{x_t + \hat{h}_t}{w+1} - \tau_t||^2 + \frac{1}{\delta_0 (w+1)^2} ||\hat{h}_t - h_t||^2,
\end{align}

Also, from Proposition~\ref{prop:prev_work_smoothness}, for any $0 < \delta_1$ we have:
\begin{align}
    \label{eq:noisy_smoothness_1}
    \frac{1}{1+\delta_1} f_t(x_t-u_t) \leq f_t(x_t - x_t) + \frac{\ell}{2\delta_1} ||u_t - x_t||^2,
\end{align}

In addition the function $\mathcal{F}_3(x) = \lambda_2||x_t - x||^2$ is $2\lambda_2$-strongly smooth, so for any $0 < \delta_2$ we have:

\begin{equation}
    \label{eq:noisy_smoothness_3}
    \frac{\lambda_2}{(1+\delta_2)}  ||x_t - x_{t-1}||^2 \leq \lambda_2 ||x_t - \hat{x}_{t-1}||^2 + \frac{\lambda_2}{\delta_2} ||\hat{x}_{t-1} - x_{t-1}||^2.
\end{equation}

By replacing \eqref{eq:noisy_smoothness_1}, \eqref{eq:noisy_smoothness_2}, and \eqref{eq:noisy_smoothness_3} into \eqref{eq:noisy_convex_phi}, we get:

\begin{align*}
    \ &\frac{1}{1+\delta_0} ||\frac{x_t + h_t}{q+1} - \tau_t||^2 + \frac{\lambda_1}{1+\delta_1} f_t(x_t-u_t) + \frac{\lambda_2}{1+\delta_2} ||x_t - x_{t-1}||^2\\
    -\ & \frac{1}{\delta_0(w+1)^2} ||\hat{h}_t - h_t||^2 - \frac{\lambda_1 \ell}{2\delta_1} ||u_t - x_t||^2 - \frac{\lambda_2}{\delta_2} ||\hat{x}_{t-1} - x_{t-1}||^2 + \frac{\eta_2}{2} ||u_t - x_t||^2\\
    \leq\ & ||\frac{\hat{x}_t + \hat{h}_t}{w+1} - \tau_t||^2 + \lambda_1 f_t(\hat{x}_t - u_t) + \lambda_2 ||\hat{x}_t - \hat{x}_{t-1}||^2.
\end{align*}

Which gives us:
\begin{align}
    \ &\frac{1}{1+\delta_0} ||\frac{x_t + h_t}{q+1} - \tau_t||^2 + \frac{\lambda_1}{1+\delta_1} f_t(x_t-u_t) + \frac{\lambda_2}{1+\delta_2}  ||x_t - x_{t-1}||^2\notag\\\
    \leq\ & ||\frac{\hat{x}_t + \hat{h}_t}{w+1} - \tau_t||^2 + \lambda_1 f_t(\hat{x}_t - u_t) + \lambda_2 ||\hat{x}_t - \hat{x}_{t-1}||^2\notag\\\
    +\ &  \frac{1}{\delta_0(w+1)^2} ||\hat{h}_t - h_t||^2 + (\frac{\lambda_1 \ell}{2\delta_1} - \frac{\eta_2}{2}) ||u_t - x_t||^2 + \frac{\lambda_2}{\delta_2} ||\hat{x}_{t-1} - x_{t-1}||^2.\notag\
\end{align}

By getting sum over different time slots from both sides and using Lemma~\ref{lem:sum_diff_ys} we get:
\begin{align}
    \ &\sum_{t=1}^T \bigg( \frac{1}{1+\delta_0} ||\frac{x_t + h_t}{q+1} - \tau_t||^2 + \frac{\lambda_1}{1+\delta_1} f_t(x_t-u_t) + \frac{\lambda_2 }{1+\delta_2} ||x_t - x_{t-1}||^2\bigg)\notag\\
    \leq\ & \sum_{t=1}^T \bigg( ||\frac{\hat{x}_t + \hat{h}_t}{w+1} - \tau_t||^2 + \lambda_1 f_t(\hat{x}_t - u_t) + \lambda_2  ||\hat{x}_t - \hat{x}_{t-1}||^2 \bigg) \notag\\
    \label{eq:cost_vs_cost_1}
    +\ &  \frac{1}{\delta_0(w+1)^2} \bigg(\sum_{t=1}^T ||\hat{h}_t - h_t||^2\bigg) + \bigg(\frac{\lambda_1 \ell}{2\delta_1} + \frac{\lambda_2 \lambda_1^2\ell^2}{\delta_2\eta^2} - \frac{\eta_2}{2}\bigg) \bigg(\sum_{t=1}^T||u_t - x_t||^2 \bigg).
\end{align}
where the inequality is derived by applying $\hat{x}_t=x(u_t,\hat{h}_t)$ and $x_t=x(x_t,\hat{h}_t)$ in Lemma~\ref{lem:sum_diff_ys}.
Combining this with Lemma~\ref{lem:sum_chi}, we also obtain:
\begin{align}
    \label{eq:diff_chi_vs_y}
    &\sum_{t=1}^T ||\hat{h}_t - h_t||^2  \leq w^2 \bigg( \sum_{t=1}^T  ||\hat{x}_{t} - x_{t}||^2 \bigg) \leq \frac{w^2\lambda_1^2\ell^2}{\eta^2} \bigg( \sum_{t=1}^T  ||u_t - x_{t}||^2 \bigg).
\end{align}
By replacing \eqref{eq:diff_chi_vs_y} into \eqref{eq:cost_vs_cost_1} we get:


\begin{align}
    &\ \min\{\frac{1}{1+\delta_0} , \frac{1}{1+\delta_1}  , \frac{1}{1+\delta_2} \}\ \text{Cost}(\algName, \mathcal{I}) \notag\\
    \label{eq:parametric_CR}
    \leq\ & \text{Cost}(\informedGreedy, \mathcal{I}) +  (\frac{\lambda_1^2\ell^2}{\delta_0\eta^2} + \frac{\lambda_1 \ell}{2\delta_1} + \frac{\lambda_2\lambda_1^2\ell^2}{\delta_2\eta^2} - \frac{\eta_2}{2}) \bigg(\sum_{t=1}^T ||u_t - x_t||^2 \bigg).
\end{align}

By selecting values for $\delta_0, \delta_1$, and $\delta_2$ as 
\begin{equation*}
    \delta_0 = \delta_1 = \delta_2 = \frac{\lambda_1\ell(\eta^2 + 2\lambda_1 \ell (1+\lambda_2))}{\eta_2 \cdot \eta^2},
\end{equation*}
the degradation factor of \algName will be upper bounded as follows:
\begin{equation*}
    \texttt{DF}(\algName, \informedGreedy) \leq 1 + \frac{\ell (\eta^2 + 2\lambda_1 \ell (1+\lambda_2))}{m\eta (\eta - m\lambda_1)}.
\end{equation*}

\end{proof}

\subsection{Proof of Theorem~\ref{thm:greedy_predicted}}
\label{app:greedy_predicted_lower_proof}

\begin{figure}[ht]
    \centering
    \includegraphics[width=0.7\linewidth]{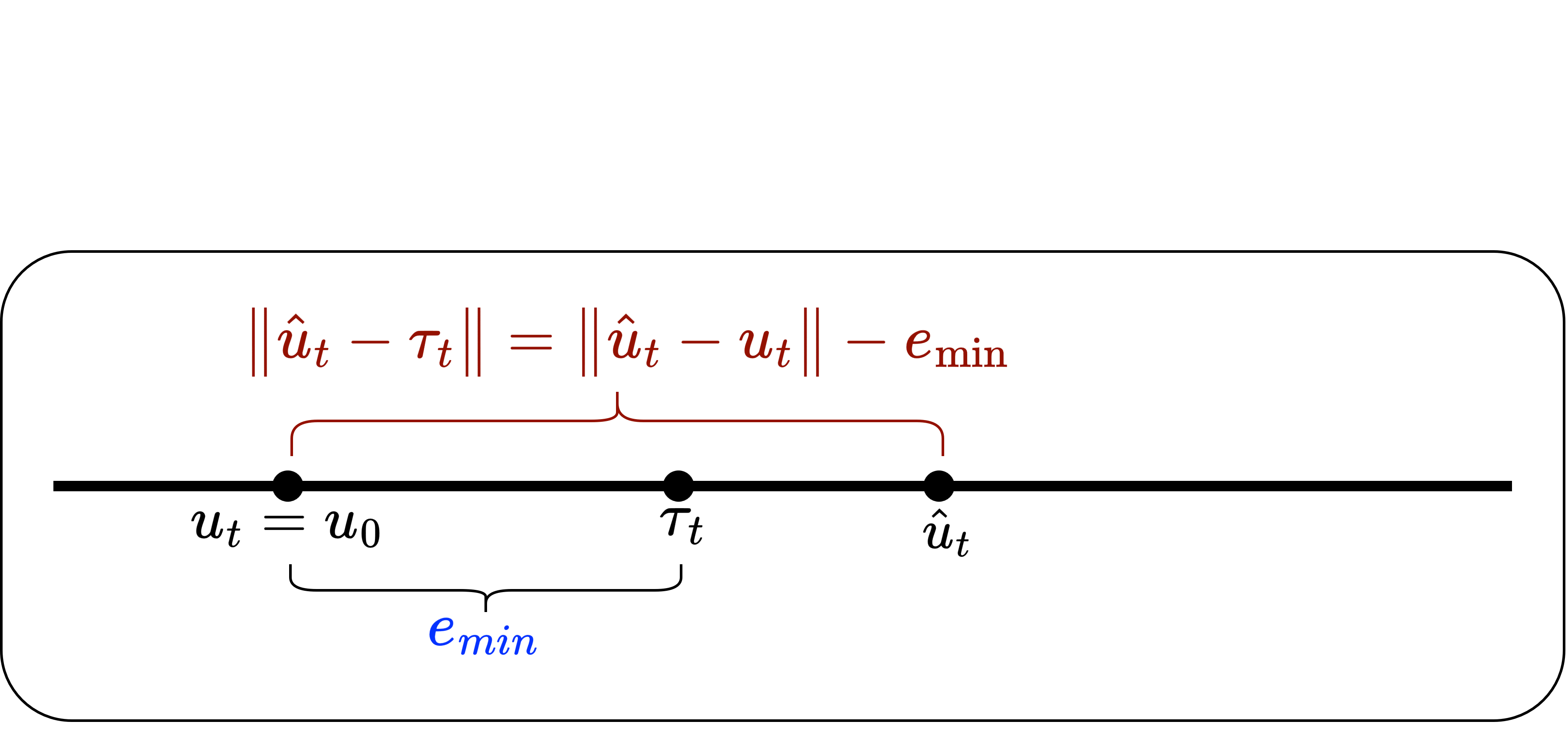}
    \caption{Coordinates of actual and predicted targets used in the proof of Theorem~\ref{thm:greedy_predicted}.}
    \label{fig:PGA_proof_fig}
\end{figure}

\begin{proof}
Let define the error of prediction of adversarial target at time step $t$ as follows:
\begin{equation*}
    e_t := ||u_t - \hat{u}_t||.
\end{equation*}

We prove this theorem by constructing a specific instance of the problem. Consider the target trajectory \( u_t \) and the adversarial target trajectory \( \tau_t \) defined as follows:
\begin{align}
    u_t &= u_0,\\
    \label{eq:greedy_proof_tau}
    \tau_t &= u_0 + e_{\min} \cdot \frac{u_0}{||u_0||},
\end{align}
where $u_0$ is an arbitrary time-independent target, and $e_{\min}$ is constant which satisfies
\begin{equation*}
    e_{\min} \leq \min_t \ e_t.
\end{equation*}

Now, suppose that the predicted value of \( u_t \) satisfies the following condition:
\begin{align}
    \label{eq:greedy_proof_uHat}
    \hat{u}_t = u_t + e_t \cdot \frac{u_0}{||u_0||};
\end{align}
see Figure~\ref{fig:PGA_proof_fig} for an illustration. 

Under this setup, the cost incurred by \informedGreedy is upper-bounded as:
\begin{align}
    \label{eq:informed_greedy_lower_cost}
    \text{Cost}(\informedGreedy, \mathcal{I}_0) \leq \lambda_1 \sum_{t=1}^T f_t (\tau_t - u_t),
\end{align}
where this bound is attained when \informedGreedy selects \( \tau_t \) at every time step.  

On the other hand, the cost incurred by \greedy satisfies the following lower bound:
\begin{align}
\label{eq:greedy_lower}
    \text{Cost}(\greedy, \mathcal{I}_0) \geq \lambda_1 \sum_{t=1}^T f_t (\tilde{x}_t - u_t). 
\end{align}

Given \eqref{eq:greedy_proof_tau} and \eqref{eq:greedy_proof_uHat}, there exists a positive constant \( \alpha_t \) such that, for every time step \( t \), we can express \( \tilde{x}_t \) as:
\begin{align}
    \tilde{x}_t = (1+\alpha_t \lambda_1) \tau_t.
\end{align}
Note that, when $\lambda_1$ gets very small values, $\tilde{x}_t$ converges to $\tau_t$. Substituting this into~\eqref{eq:greedy_lower} gives:
\begin{align}
\label{eq:greedy_lower_proof_1}
    \text{Cost}(\greedy, \mathcal{I}_0) \geq& \lambda_1 \sum_{t=1}^T f_t (\tau_t - u_t) + \frac{\lambda_1 m}{2} \ \sum_{t=1}^T \left\|e_t - e_{\min} \right\|^2.
\end{align}

Substituting \eqref{eq:informed_greedy_lower_cost} into \eqref{eq:greedy_lower_proof_1} gives:

\begin{align}
\label{eq:greedy_lower_proof_2}
    \frac{\text{Cost}(\greedy, \mathcal{I}_0)}{\text{Cost}(\informedGreedy, \mathcal{I}_0)} \geq 1 + \frac{m \sum_{t=1}^T ||e_t - e_{\min}||^2}{2\sum_{t=1}^T f_t(\tau_t - u_t)} = 1 + \frac{m \lambda_1 \sum_{t=1}^T ||e_t - e_{\min}||^2}{2\lambda_1\sum_{t=1}^T f_t(e_{\min} \cdot u_0 / ||u_0||)}.
\end{align}

and limiting $e_{\min} \to 0$ completes the proof.
\end{proof}

\subsection{Proof of Theorem~\ref{thm:learning_augmented_perf}}
\label{app:learning_augmented_proof}
\begin{proof}
We begin by analyzing the performance of \cort under fully adversarial predictions, highlighting the robustness of \cort. Let $\tilde{x}_t$ and $x_t$ denote the actions of \cort and \algName, respectively, at time step $t$.  By Proposition~\ref{prop:prev_work_smoothness} and Lemma~\ref{lem:sum_diff_ys}, for any positive parameter $\delta$, we have
\begin{align}
    \sum_{t=1}^T \left\|\frac{\tilde{x}_t + \tilde{h}_t}{w+1} - \tau_t\right\|^2 
    \leq\,& \sum_{t=1}^T ({1+\delta}) \left\|\frac{x_t + h_t}{w+1} - \tau_t\right\|^2 \notag\\
    +\,& \sum_{t=1}^T \left(1+\frac{1}{\delta}\right)\frac{1}{(w+1)^2} \left\|\tilde{x}_t + \tilde{h}_t - x_t - h_t\right\|^2 \notag\\
    \label{eq:LAGA_term1}
    \leq\,& \sum_{t=1}^T (1+\delta) \left\|\frac{x_t + h_t}{w+1} - \tau_t\right\|^2 
    + \sum_{t=1}^T \left(1+\frac{1}{\delta}\right) \left\|\tilde{x}_t - x_t\right\|^2,
\end{align}
where the last inequality uses Lemma~\ref{lem:sum_chi}. Similarly, for the regularization term, we have
\begin{align}
    \sum_{t=1}^T \lambda_2 \left\|\tilde{x}_t - \tilde{x}_{t-1}\right\|^2 
    \leq\,& \sum_{t=1}^T (1+\delta) \lambda_2 \left\|x_t - x_{t-1}\right\|^2 
    + \sum_{t=1}^T \left(1+\frac{1}{\delta}\right) \lambda_2 \left\|\tilde{x}_t - \tilde{x}_{t-1} - (x_t - x_{t-1})\right\|^2 \notag\\
    \label{eq:LAGA_term3}
    \leq\,& (1+\delta) \sum_{t=1}^T \lambda_2 \left\|x_t - x_{t-1}\right\|^2 
    + 4(1+\frac{1}{\delta}) \lambda_2 \sum_{t=1}^T \left\|\tilde{x}_t - x_t\right\|^2.
\end{align}

Since $f_t(\cdot)$ is $\ell$-strongly smooth, we have
\begin{align}
\label{eq:LAGA_term2}
    \lambda_1 f_t(\tilde{x}_t - u_t) 
    \leq\, & \lambda_1 f_t(x_t - u_t) + \frac{\ell \lambda_1}{2} \left\|\tilde{x}_t - x_t\right\|^2 
    + \lambda_1 \nabla f_t(x_t - u_t) \cdot (\tilde{x}_t - x_t).
\end{align}

By combining \eqref{eq:LAGA_term1}, \eqref{eq:LAGA_term3}, and \eqref{eq:LAGA_term2}, for any instance input $\mathcal{I}$, we obtain
\begin{align*}
    \text{Cost}(\cort, \mathcal{I}) 
    \leq\, & (1+\delta)\ \text{Cost}(\algName, \mathcal{I}) 
    + \left(1+\frac{1}{\delta}\right)\left(1+4\lambda_2+\frac{\ell \lambda_1}{2}\right) \sum_{t=1}^T \left\|\tilde{x}_t - x_t\right\|^2 \notag \\
    +\,& \lambda_1 \sum_{t=1}^T \nabla f_t(x_t - u_t) \cdot (\tilde{x}_t - x_t).
\end{align*}

Moreover, since $f_t(\cdot)$ is $\ell$-strongly smooth, it follows that
\begin{align}
    \left\|\nabla f_t(x_t - u_t)\right\| \leq \ell \left\|x_t - u_t\right\|,
\end{align}
which implies
\begin{align}
    \text{Cost}(\cort, \mathcal{I}) 
    \leq\, & (1+\delta)\ \text{Cost}(\algName, \mathcal{I}) 
    + \left(1+\frac{1}{\delta}\right)\left(1+4\lambda_2+\frac{\ell \lambda_1}{2}\right) \sum_{t=1}^T \left\|\tilde{x}_t - x_t\right\|^2 \notag\\
    +\,& \lambda_1 \ell \sum_{t=1}^T \left\|x_t - u_t\right\| \cdot \left\|\tilde{x}_t - x_t\right\| \notag\\
    \leq\, & (1+\delta)\ \text{Cost}(\algName, \mathcal{I}) 
    + \left(1+\frac{1}{\delta}\right)\left(1+4\lambda_2+\frac{\ell \lambda_1}{2}\right) \sum_{t=1}^T \left\|\tilde{x}_t - x_t\right\|^2 \notag\\
    \label{eq:upper_cost_laga}
    +\,& \lambda_1 \ell \sum_{t=1}^T \left[\frac{1}{\alpha} \left\|x_t - u_t\right\|^2 + \alpha \left\|\tilde{x}_t - x_t\right\|^2\right],
\end{align}
where $\alpha$ is an arbitrary positive constant. In addition, for \algName we have:
\begin{align}
    \label{eq:best_lower_vs_xu}
    \text{Cost}(\algName, \mathcal{I}) \geq \lambda_1 \sum_{t=1}^T f_t(x_t - u_t) \geq \frac{m\lambda_1}{2} \sum_{t=1}^T ||x_t - u_t||^2.
\end{align}

Also, based on Lemma~\ref{lem:sum_diff_ys}, we have
\begin{align}
    \sum_{t=1}^T \left\|\tilde{x}_t - x_t\right\|^2 
    \leq\, & (\frac{\lambda_1 \ell}{\eta})^2 \sum_{t=1}^T \left\|x_t - \tilde{u}_t\right\|^2 
    \leq (\frac{\lambda_1 \ell\theta}{\eta})^2\sum_{t=1}^T D_t^2 \notag \\
    \label{eq:diff_laga_best}
    \leq\, & (\frac{\lambda_1 \ell \theta}{\eta})^2  \sum_{t=1}^T \left\|x_t - u_t\right\|^2 \leq \frac{2\lambda_1 \ell^2}{m \eta^2}\ \theta^2\ \text{Cost}(\algName, \mathcal{I}),
\end{align}

By combining~\eqref{eq:diff_laga_best}, \eqref{eq:best_lower_vs_xu}, and~\eqref{eq:upper_cost_laga}, we obtain
\begin{align}
    \text{Cost}(\cort, \mathcal{I}) 
    \leq\; & \bigg[1+\delta + \frac{\lambda_1}{\alpha} + \bigg( (1+\frac{1}{\delta})(1+4\lambda_2 + \frac{\ell \lambda_1}{2}) + \alpha \lambda_1 \ell (\frac{\lambda_1 \ell}{\eta})^2 \bigg) \theta^2 \bigg] \text{Cost}(\algName, \mathcal{I}),
\end{align}

Moreover, as $\theta \to 0$, \cort converges to \algName. This implies
\begin{align}
    \label{eq:laga_robustness}
    \texttt{DF}(\cort, \informedGreedy) 
    \leq \texttt{DF}(\algName, \informedGreedy)\left(1 + \theta^2\mathcal{O}(1)\right).
\end{align}

Next, we proceed to analyze the performance of \cort under perfect prediction (Consistency analysis). 

Let $x_t$ denote the action of \informedGreedy at time step $t$. Under perfect prediction conditions, we have $\hat{u}_t = \tilde{u}_t$. In such a case, if $\left\|u_t - x_t\right\| \leq D_t$, the actions of \informedGreedy and \cort coincide.  Thus, in order to maximize the gap between the performance of \cort and \informedGreedy, in the worst case scenario, $\mathcal{I}_{worst}$, an adversary must select targets such that the following inequality holds:
\begin{align}
    D_t \leq \left\|u_t - x_t\right\|.
\end{align}

Based on Assumption~\ref{asmp:start_history}, $u_0$ and $x_0$ are identical initially, implying $D_1 = 0$. Combining this with the above inequality, we conclude that, in the worst-case scenario, the following relation holds:
\begin{align}
\label{eq:D_t}
    D_t = \left\|u_{t-1} - x_{t-1}\right\| \leq \left\|u_t - x_t\right\| = D_{t+1}.
\end{align}

Furthermore, by the definitions of $x_t$, $\tilde{x}_t$, $\hat{x}_t$, and $u_t$, these points lie along a direct line segment. Consequently, there exist constants $\beta_t \in [0,1]$ such that
\begin{align}
    \tilde{x}_t =\ & \beta_t \hat{x}_t + (1-\beta_t) x_t.
\end{align}

Based on this, using the convexity of cost terms we get:

\begin{align}
    \text{Cost}_t(\cort) =\ &  ||\frac{\tilde{x}_t + \tilde{h}_t}{w+1} - \tau_t||^2 + \lambda_1 f_t(\tilde{x}_t - u_t) + \lambda_2 ||\tilde{x}_t - \tilde{x}\notag_{t-1}||^2\\
    \leq\ & {\beta} ||\frac{\hat{x}_t + \hat{h}_t}{w+1} - \tau_t||^2 + ({1-\beta_t}) ||\frac{{x}_t + {h}_t}{w+1} - \tau_t||^2  \notag \\
    +\ & \beta_t \lambda_1 f_t(\hat{x}_t - u_t) + (1-\beta_t) \lambda_1 f_t(x_t - u_t) \notag \\
    +\ & {\lambda_2 \ \beta_t} ||\hat{x}_t - \tilde{x}_{t-1}||^2 + {\lambda_2  (1-\beta_t)} ||x_t - \tilde{x}_{t-1}||^2 \notag \\
    \leq\ & {\beta_t} ||\frac{\hat{x}_t + \hat{h}_t}{w+1} - \tau_t||^2 + ({1-\beta_t}) ||\frac{{x}_t + {h}_t}{w+1} - \tau_t||^2  \notag \\
    +\ & \beta_t \lambda_1 f_t(\hat{x}_t - u_t) + (1-\beta_t) \lambda_1 f_t(x_t - u_t) \notag \\
    +\ & {\lambda_2 \ \beta_t} ||\hat{x}_t - \hat{x}_{t-1}||^2 + \lambda_2  (1-\beta_{t}) ||x_t - x_{t-1}||^2 \notag\\
    \label{eq:cost_laga_robustness_1}
    +\ & {\lambda_2  \beta_t (\frac{\lambda_1 \ell}{\eta})^2} (D_t - \theta D_{t-1})^2 + {\lambda_2  (1-\beta_t) } (\frac{\lambda_1 \ell}{\eta})^2 \theta^2 D_t^2,
\end{align}
where the last inequality holds only for the worst-case instance $\mathcal{I}_{\text{worst}}$, using the fact that, by definition, the following property holds for $\mathcal{I}_{\text{worst}}$:
\begin{align}
    \label{eq:consistency_D_t_dynamic}
    \theta D_t\ \leq\ & D_{t+1}, \quad \forall t\\
    ||\tilde{u}_t - x_t||\ =\ & \theta D_t,\quad \forall t\\
    ||u_t - \tilde{u}_t||\ =\ &  D_{t+1} - \theta D_{t},\quad \forall t\\
    ||u_t - x_t||\ =\ &  D_{t+1}.\quad \forall t
\end{align}



This yields:
\begin{align}
    \text{Cost}_t(\cort, \mathcal{I}_{worst}) \leq\ & \beta_t\ \text{Cost}_t(\informedGreedy, \mathcal{I}_{worst})\notag\\
    +\ & (1-\beta_t)\ \text{Cost}_t(\algName, \mathcal{I}_{worst}) \notag \\
    \label{eq:cost_laga_consistency_1}
    +\ & \lambda_2 (\frac{\lambda_1 \ell}{ \eta})^2\bigg[ D_t^2 \bigg(\beta_t(1-\frac{\theta D_{t-1}}{D_t})^2 + (1-\beta_t)\theta^2\bigg) \bigg].
\end{align}

In addition, we can provide upper bounds on the values of $\beta_t$ and $1 - \beta_t$ as follows:
\begin{align}
    \label{eq:beta_bound}
    \beta_t &= \frac{\|\tilde{x}_t - x_t\|}{\|\hat{x}_t - x_t\|} 
    \leq \frac{\lambda_1 \ell}{\eta} \cdot \frac{\eta}{m \lambda_1} \cdot \frac{\theta D_t}{D_{t+1}} 
    = \frac{\ell}{m} \cdot \frac{\theta D_t}{D_{t+1}}, \\
    \label{eq:neg_beta_bound}
    1 - \beta_t &= \frac{\|\tilde{x}_t - \hat{x}_t\|}{\|\hat{x}_t - x_t\|} 
    \leq \frac{\ell}{m} \cdot \left( \frac{D_{t+1} - \theta D_t}{D_{t+1}} \right) = \frac{\ell}{m} (1- \frac{\theta D_t}{D_{t+1}}),
\end{align}
where we used the convexity of the cost function and Lemma~\ref{lem:sum_diff_ys} to derive these bounds. These expressions reveal that when $\theta$ is small (i.e., $\theta \to 0$), $\beta_t$ also becomes small, indicating that the action of \cort\ closely follows that of \algName. Conversely, as $\theta$ grows large (i.e., $\theta \to \infty$), $\beta_t$ approaches 1, and the action of \cort\ becomes similar to that of \informedGreedy.

Also, since \algName is minimizing the cost value ignoring the adversarial cost at time step $t$, the cost of \informedGreedy in the worst case instance is lower bounded as follows:
\begin{align}
    \text{Cost}(\informedGreedy, \mathcal{I}_{worst}) \geq\ & \sum_{t=1}^T \bigg( ||\frac{x_t + \hat{h}_t}{w+1} - \tau_t||^2 + \lambda_2 ||x_t - \hat{x}_{t-1}||^2 \bigg) \notag \\
    +\ & (\frac{\eta-\lambda_1 m}{2}) (\frac{m\lambda_1}{\eta})^2 \sum_{t=1}^T||x_t - u_t||^2 +  \frac{m\lambda_1}{2} (\frac{\eta-m\lambda_1}{\eta})^2 \sum_{t=1}^T||x_t - u_t||^2 \notag \\
    \geq\ & \sum_{t=1}^T \bigg( ||\frac{x_t + \hat{h}_t}{w+1} - \tau_t||^2 + \lambda_2 ||x_t - \hat{x}_{t-1}||^2 \bigg) \notag \\
    +\ & \frac{m\lambda_1}{2\eta}(\eta-m\lambda_1) \sum_{t=1}^T ||x_t - u_t||^2 \notag\\
    \geq\ & \frac{m\lambda_1}{2\eta}(\eta-m\lambda_1) \sum_{t=1}^T D_{t+1}^2 \notag\\
    \label{eq:cost_IGA_consistency_1}
    =\ & \frac{m\lambda_1}{2\eta}(\eta-m\lambda_1) \sum_{t=1}^{T+1} D_{t}^2,
\end{align}
where in the last inequality we used the fact that $D_1$ = 0. Combining \eqref{eq:cost_laga_consistency_1} and \eqref{eq:cost_IGA_consistency_1} yields:


\begin{align}
    \frac{\text{Cost}_t(\cort, \mathcal{I}_{worst})}{\text{Cost}_t(\informedGreedy, \mathcal{I}_{worst})} \leq\ & \beta_t + (1-\beta_t) \frac{\text{Cost}_t(\algName, \mathcal{I}_{worst})}{\text{Cost}_t(\informedGreedy, \mathcal{I}_{worst})}\notag\\
    +\ & \frac{2\lambda_2  \lambda_1\ell^2 \bigg[\sum_{t=1}^T D_t^2 \bigg(\beta_t(1-\frac{\theta D_{t-1}}{D_t})^2 + (1-\beta_t)\theta^2\bigg) \bigg]}{m\eta (\eta- m\lambda_1) \sum_{t=1}^{T+1} D_{t}^2},
\end{align}

Note that the latter term increases with $\theta$, and both the numerator and denominator grow at most quadratically with respect to $\theta$. Its maximum value, as $\theta \to \infty$, is bounded by:

\begin{align}
    \left[\frac{2\lambda_2  \lambda_1\ell^2 \bigg[\sum_{t=1}^T D_t^2 \bigg(\beta_t(1-\frac{\theta D_{t-1}}{D_t})^2 + (1-\beta_t)\theta^2\bigg) \bigg]}{m\eta (\eta- m\lambda_1) \sum_{t=1}^{T+1} D_{t}^2} \bigg | \theta\to \infty \right]\notag \\
    \leq\ \frac{2\lambda_2  \lambda_1\ell^2 \sum_{t=1}^T D_t^2 \theta^2 }{m\eta (\eta- m\lambda_1) \sum_{t=1}^{T+1} D_{t}^2} \leq \frac{2\lambda_2  \lambda_1\ell^2 }{m\eta (\eta- m\lambda_1)}.
\end{align}

Finally, the proof follows by noting that $\beta_t$ increases with $\theta$, converging to $0$ as $\theta \to 0$, and approaching $1$ as $\theta \to \infty$.

\end{proof}

\section{Additional Details of Experiments}
\label{app:experimental_detail}
In this section, we provide additional details of the experimental setup.

\subsection{Result of Experiments on Impact of $\lambda_2$, and $\tau_t$}

Figure~\ref{fig:elastic_l2} illustrates the impact of the switching cost coefficient $\lambda_2$ on the total cost incurred by the algorithms. The results show that $\lambda_2$ influences the cost functions in a manner similar to the weight parameter $w$. As $\lambda_2$ increases, both online algorithms and the offline optimal algorithm are more heavily penalized for making large changes between consecutive actions. This discourages frequent switching, leading to smoother action sequences. Consequently, the adversarial cost component contributes less to the overall cost, resulting in reduced total cost values.

We also evaluate the impact of the daily average value of $\tau_t$ on algorithm performance. Following the structure described in Section~\ref{sec:experiments}, we vary $\tau_t$ across the day by modifying its value during mid-peak periods and then shifting it by $+0.1$ (i.e., 10\%) during off-peak and $-0.1$ (i.e., 10\%) during on-peak hours. This setup ensures that the daily average of $\tau_t$ matches its value during mid-peak periods. Results of this analysis, shown in Figure~\ref{fig:elastic_tau}, indicate that the effect of the daily average of $\tau_t$ on the normalized cost is modest compared to other parameters like $\lambda_1$, $\lambda_2$, and $w$. This limited sensitivity is intuitive, as we preserve the shape of the $\tau_t$ variation pattern throughout the day and only apply a uniform shift. Note that this analysis focuses solely on the impact of daily average $\tau_t$ on algorithm cost; exploring its influence on other metrics—such as the average allocation to elastic or inelastic workloads—is left for future work.

\begin{figure}[t]
  \centering
  \begin{subfigure}[t]{0.48\linewidth}
    \centering
    \includegraphics[width=\linewidth]{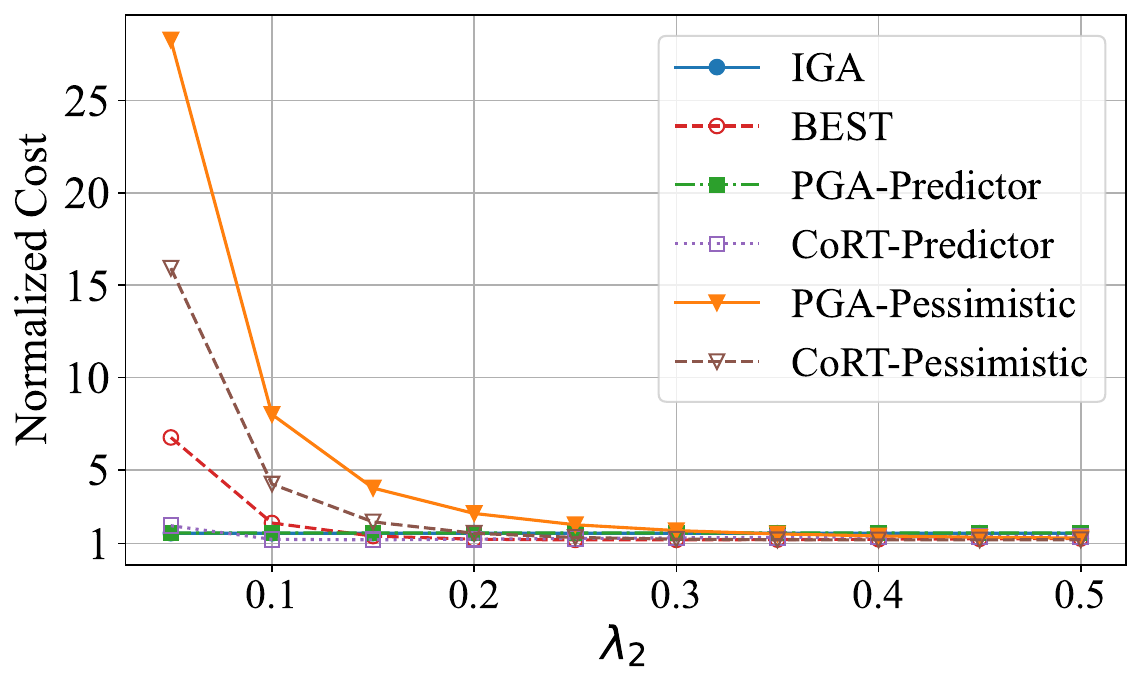}
    \caption{ Impact of parameter $\lambda_2$}
    \label{fig:elastic_l2}
  \end{subfigure}
  \hfill
  \begin{subfigure}[t]{0.48\linewidth}
    \centering
    \includegraphics[width=\linewidth]{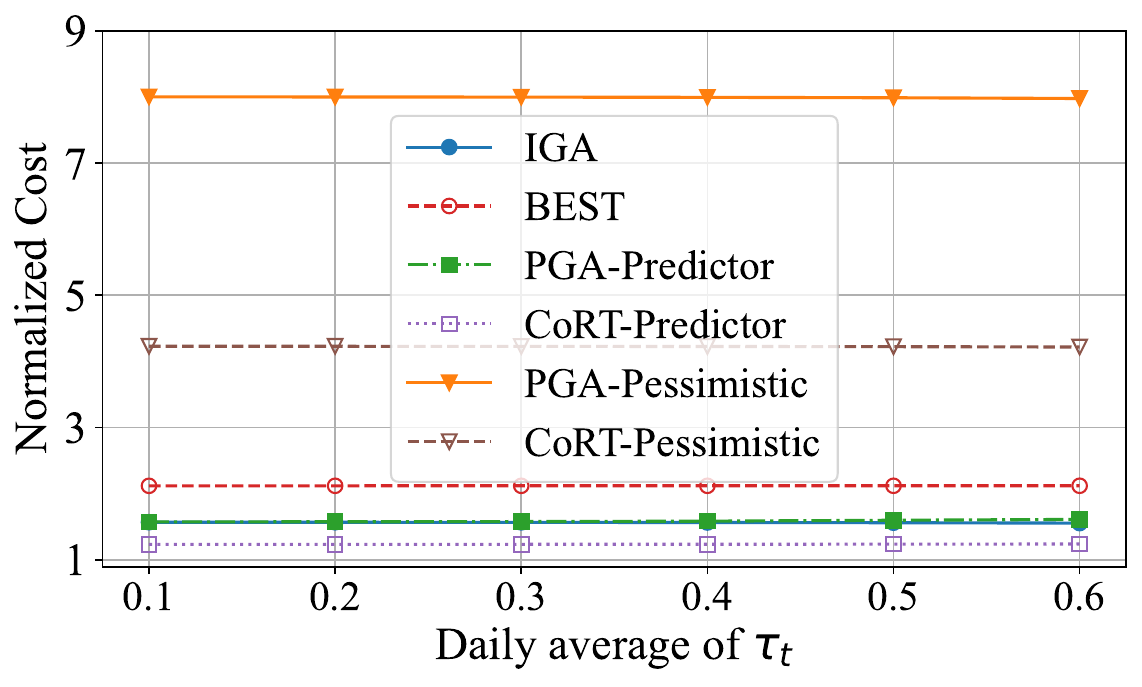}
    \caption{Impact of daily average of trajectory targets}
    \label{fig:elastic_tau}
  \end{subfigure}
  \vspace{-0.5em}
    \caption{Impact of $\lambda_2$ and the daily average of trajectory targets, $\tau_t$, on algorithm cost. While $\lambda_2$ significantly affects the normalized cost of the algorithms, the daily average of $\tau_t$ has a minimal impact.}

  \label{fig:elastic_fig2}
\end{figure}

\subsection{More Detail on the LSTM Predictor Used in Section~\ref{sec:experiments}}

\begin{figure}[!h]
    \centering
    \includegraphics[width=0.75\linewidth]{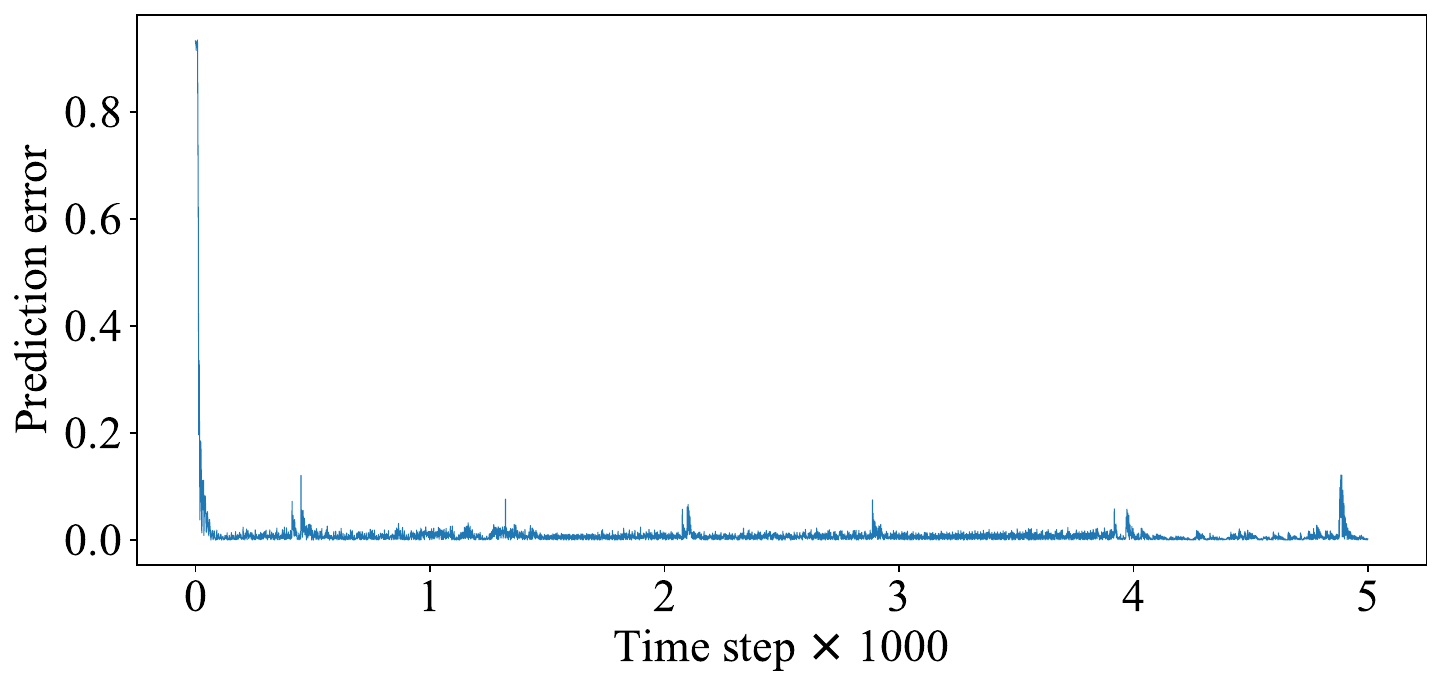}
    \caption{Prediction error $\|u_t - \hat{u}_t\|$ over time steps.}
    \label{fig:LSTM_error}
\end{figure}

To estimate the adversary’s target $u_t$ at each time step in an online fashion, we implement an LSTM-based regression model that learns the temporal dependencies in the observed sequence of $u$ values. Specifically, we train a one-layer Long Short-Term Memory (LSTM) network followed by a fully connected linear layer. The LSTM model receives a sliding window of the previous $W$ observations $\{u_{t-W}, \ldots, u_{t-1}\}$ and predicts the next value $\hat{u}_t$.

Our architecture consists of:

\textbf{Input layer:} A sequence of $W=10$ scalar values, each representing the observed $u_t$ at previous time steps. 

\noindent \textbf{LSTM layer:} A single-layer LSTM with hidden size 32, which processes the input sequence and outputs a hidden state vector representing the temporal features of the sequence.
 
 \noindent \textbf{Output layer:} A linear layer of size $32 \rightarrow 1$ that maps the last hidden state to the final prediction $\hat{u}_t$.

We train the model incrementally in an online manner, using a single gradient update per time step. The model is optimized using the Adam optimizer with a learning rate of $10^{-2}$. The training is performed in real-time as new data arrives, making the approach suitable for dynamic and non-stationary environments.

Figure~\ref{fig:LSTM_error} shows the prediction error ($\|u_t - \hat{u}_t\|$) over time for the first 5,000 steps. The results demonstrate that the LSTM network achieves a high level of accuracy in predicting $u_t$. Specifically, the average prediction error across the entire horizon is $0.01$, with a standard deviation of $0.04$. Owing to this high accuracy, the performance of \greedy-Predictor\ and \cort-Predictor\ closely matches that of \greedy-Optimistic\ and \cort-Optimistic\ reported in Section~\ref{sec:experiments}.








\end{document}